%% file: arxiv.tex
\documentclass{article}

\usepackage[preprint]{neurips_2026}

\usepackage[utf8]{inputenc} 
\usepackage[T1]{fontenc}    
\usepackage{url}            
\usepackage{booktabs}       
\usepackage{amsfonts}       
\usepackage{nicefrac}       
\usepackage{microtype}      
\usepackage{xcolor}         

\usepackage{colortbl}

\usepackage{amsmath}
\usepackage{amssymb}
\usepackage{mathtools}
\usepackage{amsthm}
\usepackage{wrapfig}
\usepackage{listings}
\usepackage{xcolor}
\usepackage{colortbl}
\usepackage{tabulary}
\usepackage{etoolbox}
\usepackage{multirow}
\usepackage{enumitem}
\usepackage{xspace}
\usepackage{algorithm,algorithmicx,algpseudocode}
\usepackage{subfig}

\usepackage[pagebackref=true,breaklinks=true,colorlinks,urlcolor={red!90!black},linkcolor={red!90!black},citecolor={blue!90!black},bookmarks=false]{hyperref}
\usepackage[toc,page,header]{appendix}
\usepackage{minitoc}

\newtheorem{Theorem}{Theorem}[section]

\newtheorem{Proposition}[Theorem]{Proposition}
\newtheorem{Assumption}[Theorem]{Assumption}
\newtheorem{Lemma}[Theorem]{Lemma}

\theoremstyle{definition}
\newtheorem{Remark}[Theorem]{Remark}

\newcommand{\tablestyle}[2]{\setlength{\tabcolsep}{#1}\renewcommand{\arraystretch}{#2}\centering
\small
}
\newcolumntype{x}[1]{>{\centering\arraybackslash}p{#1pt}}
\newcolumntype{y}[1]{>{\raggedright\arraybackslash}p{#1pt}}
\newcolumntype{z}[1]{>{\raggedleft\arraybackslash}p{#1pt}}

\newcommand\tablefontsize{%
  \fontsize{8pt}{9pt}\selectfont
}

\newcommand{\method}{W-Flow\xspace}
\newcommand{\OT}{\mathrm{OT}}

\newcommand{\1}{\mathbf{1}}

\newcommand{\black}[1]{{\textcolor{black}{#1}}}

\title{One-Step Generative Modeling via \\Wasserstein Gradient Flows}

\author{
  \centering
  Jiaqi Han$^{1}\footnotemark[1]$
  \And
  Puheng Li$^{1}$\thanks{Equal contribution. Correspondence to: \href{mailto:jiaqihan@stanford.edu}{\texttt{jiaqihan@stanford.edu}}, \href{mailto:puhengli@stanford.edu}{\texttt{puhengli@stanford.edu}}.} \And
  Qiushan Guo$^2$ \AND
  Renyuan Xu$^1$\footnotemark[2] \And
  Stefano Ermon$^1$\footnotemark[2] \And 
  Emmanuel J. Cand\`es$^1$\thanks{Equal advising.} \AND
  \\[-1em] 
  \centerline{$^1$Stanford University \qquad $^2$ByteDance
  }
\\[1em]
  \centerline{\texttt{Project Page:}~\url{https://hanjq17.github.io/W-Flow/}}
}

\input{math_commands}

\begin{document}

\maketitle
\doparttoc 
\faketableofcontents 

\begin{abstract}

{\color{black}Diffusion models and flow-based methods have shown impressive generative
capability, especially for images, but their sampling is expensive because it requires
many iterative updates. We introduce~\emph{\method}, a framework for training a
generator that transforms samples from a simple reference distribution into
samples from a target data distribution in a single step. This is achieved in two steps: we first define an evolution from the reference
distribution to the target distribution through a Wasserstein gradient flow that
minimizes an energy functional; second, we train a static neural
generator to compress this evolution into one-step generation. We instantiate the energy functional with the Sinkhorn divergence, which yields an efficient
optimal-transport-based update rule that captures global distributional
discrepancy and improves coverage of the target distribution. We further prove that the
finite-sample training dynamics converge to the 
continuous-time distributional dynamics under suitable assumptions.} Empirically, \method sets a new state of the art for one-step ImageNet 256$\times$256 generation, {\color{black}achieving} 1.29 FID, with improved mode coverage and domain transfer. Compared to multi-step diffusion models with similar FID scores, our method yields approximately 100$\times$ faster sampling. These results {\color{black}show} that Wasserstein gradient flows provide a principled and effective foundation for fast and high-fidelity generative modeling.

\end{abstract}

\input{sec/intro}

\input{sec/related_work}

\input{sec/method}

\input{sec/experiments}

\input{sec/conclusion}

\section*{Acknowledgments}
We gratefully acknowledge the Stanford Marlowe cluster~\cite{kapfer2025marlowe} for supporting part of the GPU computing resources. RX was partially supported by an NSF CAREER Award DMS-2614933 and a gift fund from Point72. SE was supported by ARO
(W911NF-21-1-0125), ONR (N00014-23-1-2159), and the CZ Biohub. EC was supported by ONR (N00014-24-1-2305) and the NIH (1R01AG08950901A1). The authors thank Yaniv Romano, Chieh-Hsin Lai, Gabriel Peyr\'e, Mingyang Deng, and Tudor Manole for feedback on an early version of this paper.

\bibliographystyle{plain}
\bibliography{ref}


\newpage
\appendix

{
\hypersetup{linkcolor=black}
\addcontentsline{toc}{section}{Appendix} 
\part{Appendix} 
\parttoc 
}

\section{Additional discussions}
\subsection{Wasserstein gradient flows of energy functionals}\label{app:other}

In this section, we derive the WGFs induced by the functionals listed in Table~\ref{tab:wgf_functionals}. For MMD and KL divergence, we also provide their particle-based estimators in the same spirit as
the Sinkhorn estimator used in Sec.~\ref{subsec:particle}. The estimators are then used in the ablation study in Table~\ref{subtab:ablation-divergence}.

Throughout this section, $p$ denotes the target data distribution. Given an energy functional $\mathcal F$, recall that the corresponding
WGF follows
\begin{equation*}
    \partial_t q_t
    =
    \nabla \cdot \left(q_t \nabla \frac{\delta \mathcal F}{\delta q}(q_t)\right),
\end{equation*}
or equivalently, through the continuity equation, we have
\begin{equation*}
    \partial_t q_t + \nabla \cdot (q_t V_t) = 0,
    \qquad
    V_t(x) = - \nabla \frac{\delta \mathcal F}{\delta q}(q_t)(x).
\end{equation*}

\paragraph{Gradient flow of Sinkhorn divergence.}
To compute the first variation of $S_\varepsilon(q,p)$ under quadratic cost with respect to $q$, we first consider the dual formulation of the EOT problem:

\begin{align}
\label{eq:EOT_dual}
    \nonumber
    \mathrm{OT}_\varepsilon(q,p) = \max_{u,v} &\int_{\mathbb{R}^d} u(x) dq(x) + \int_{\mathbb{R}^d} v(y) dp(y)\\ &- \varepsilon \iint_{\mathbb{R}^d \times \mathbb{R}^d} \exp\left(\frac{u(x)+v(y)-\frac{1}{2}\|x-y\|^2}{\varepsilon}\right) dq(x) dp(y) + \varepsilon,
\end{align}
where $u(x)$ and $v(y)$ are the dual potentials.

By the envelope theorem, the first variation of the EOT cost Eq.~\eqref{eq:EOT_dual} with respect to the marginal distribution $q$ is exactly the optimal dual potential $u_{q,p}^\varepsilon(x)$:
\begin{equation*}
    \frac{\delta \mathrm{OT}_\varepsilon(q,p)}{\delta q}(x) = u_{q,p}^\varepsilon(x).
\end{equation*}

Consequently, the first variation of $S_\varepsilon(\cdot, p)$ evaluated at $q$ is:
\begin{equation*}
    \frac{\delta S_\varepsilon(\cdot, p)}{\delta q}(q)(x) = u_{q,p}^\varepsilon(x) - u_{q,q}^\varepsilon(x).
\end{equation*}

From the optimality conditions of the dual problem, $u_{q,p}^\varepsilon(x)$ satisfies the following structural equation:
\begin{equation}\label{eq:sinkhorn_dual}
    u_{q,p}^\varepsilon(x) = -\varepsilon \log \int_{\mathbb{R}^d} \exp\left(\frac{v_{q,p}^\varepsilon(y) - \frac{1}{2}\|x-y\|^2}{\varepsilon}\right) dp(y).
\end{equation}
Taking the gradient with respect to $x$ yields:
\begin{align}
    \nabla_x u_{q,p}^\varepsilon(x) = &\frac{\int_{\mathbb{R}^d} \nabla_x \frac{1}{2}\|x-y\|^2 \exp\left(\frac{v_{q,p}^\varepsilon(y) - \frac{1}{2}\|x-y\|^2}{\varepsilon}\right) dp(y)}{\int_{\mathbb{R}^d} \exp\left(\frac{v_{q,p}^\varepsilon(y) - \frac{1}{2}\|x-y\|^2}{\varepsilon}\right) dp(y)}\nonumber\\
    =& \int_{\mathbb{R}^d} \nabla_x \frac{1}{2}\|x-y\|^2 \pi_{q,p}^\varepsilon(dy|x)\nonumber\\
    =&\int_{\mathbb{R}^d} (x - y) \pi_{q,p}^\varepsilon(dy|x) \nonumber\\
    =& x - T_{q,p}^\varepsilon(x).\label{eq:dual_grad_representation}
\end{align}

Combining the previous results, we obtain
\begin{align*}
    V_{q_t,p}^\varepsilon(x) &= -\nabla_x \frac{\delta S_\varepsilon(\cdot, p)}{\delta q}(q_t)(x) \nonumber \\
    &= -\nabla_x \left( u_{q_t,p}^\varepsilon(x) - u_{q_t,q_t}^\varepsilon(x) \right) \nonumber \\
    &= -\left( x - T_{q_t,p}^\varepsilon(x) \right) + \left( x - T_{q_t,q_t}^\varepsilon(x) \right) \nonumber \\
    &= T_{q_t,p}^\varepsilon(x) - T_{q_t,q_t}^\varepsilon(x).
\end{align*}

This recovers the Sinkhorn divergence velocity field reported in Table~\ref{tab:wgf_functionals}.

\paragraph{Gradient flow of MMD.}
Let $k:\mathbb R^d \times \mathbb R^d \to \mathbb R$ be a differentiable positive definite
kernel. For probability measures $p,q\in \mathcal{P}(\mathbb{R}^d)$, we consider the energy
\begin{equation}\label{eq:MMD_energy}
    \mathcal F_{\mathrm{MMD}}(q)
    :=
    \frac12 \mathrm{MMD}^2(q,p),
\end{equation}
where
\begin{equation}\label{eq:MMD}
    \mathrm{MMD}^2(q,p)
    =
    \iint k(x,y) q(dx) q(dy)
    -
    2 \iint k(x,y) q(dx) p(dy)
    +
    \iint k(x,y) p(dx) p(dy).
\end{equation}
Equivalently,
\begin{equation*}
    \mathcal F_{\mathrm{MMD}}(q)
    =
    \frac12
    \iint k(x,y) (q-p)(dx)(q-p)(dy).
\end{equation*}
To compute the first variation, consider a perturbation
$q_s = q + s\,\chi$ with signed measure $\chi$ satisfying $\int \chi(dx)=0$. Then
\begin{align*}
    \frac{d}{ds}\mathcal F_{\mathrm{MMD}}(q_s)\bigg|_{s=0}
    &=
    \iint k(x,y) \chi(dx) (q-p)(dy)  \\
    &=
    \int
    \left[
        \int k(x,y) (q-p)(dy)
    \right]\chi(dx).
\end{align*}
Therefore, up to an additive constant,
\begin{equation*}
    \frac{\delta \mathcal F_{\mathrm{MMD}}}{\delta q}(q)(x)
    =
    \int k(x,y)q(dy)
    -
    \int k(x,y)p(dy).
\end{equation*}
Taking the negative spatial gradient gives the Wasserstein velocity field
\begin{align*}
    V_t^{\mathrm{MMD}}(x)
    &=
    -\nabla_x
    \frac{\delta \mathcal F_{\mathrm{MMD}}}{\delta q}(q_t)(x) \\
    &=
    \int \nabla_x k(x,y) p(dy)
    -
    \int \nabla_x k(x,y) q_t(dy),
\end{align*}
with $q_t$ the corresponding flow under the initial condition $q_0$.
This is the MMD velocity field reported in Table~\ref{tab:wgf_functionals}.

\paragraph{Particle-based estimator for MMD.}
Given generated particles $\{x_i\}_{i=1}^N \sim q_t$ and target particles
$\{y_j\}_{j=1}^M \sim p$, define the empirical measures
\begin{equation*}
    \widehat q_t = \sum_{i=1}^N a_i \delta_{x_i},
    \qquad
    \widehat p = \sum_{j=1}^M b_j \delta_{y_j},
    \qquad
    \sum_i a_i = \sum_j b_j = 1.
\end{equation*}
As in the Sinkhorn self-transport estimator, we  draw an independent second batch
$\{x'_\ell\}_{\ell=1}^{N'} \sim q_t$ and define
\begin{equation*}
    \widehat q'_t = \sum_{\ell=1}^{N'} a'_\ell \delta_{x'_\ell}.
\end{equation*}
Then the particle estimator of the MMD-induced velocity at $x_i$ is
\begin{equation*}
    \widehat V_{\widehat q_t,\widehat p}^{\mathrm{MMD}}(x_i)
    =
    \sum_{j=1}^M b_j \nabla_x k(x_i,y_j)
    -
    \sum_{\ell=1}^{N'} a'_\ell \nabla_x k(x_i,x'_\ell).
\end{equation*}
For uniform weights, this becomes
\begin{equation*}
    \widehat V_{\widehat q_t,\widehat p}^{\mathrm{MMD}}(x_i)
    =
    \frac1M \sum_{j=1}^M \nabla_x k(x_i,y_j)
    -
    \frac1{N'} \sum_{\ell=1}^{N'} \nabla_x k(x_i,x'_\ell).
\end{equation*}
If $k(\cdot,\cdot)$ is the Gaussian kernel
\begin{equation*}
    k_\sigma(x,y)
    =
    \exp\left(-\frac{\|x-y\|^2}{2\sigma^2}\right),
\end{equation*}
then
\begin{equation*}
    \nabla_x k_\sigma(x,y)
    =
    \frac{y-x}{\sigma^2} k_\sigma(x,y),
\end{equation*}
and hence
\begin{equation*}
    \widehat V_{\widehat q_t,\widehat p}^{\mathrm{MMD}}(x_i)
    =
    \frac1{M\sigma^2}
    \sum_{j=1}^M
    (y_j-x_i) k_\sigma(x_i,y_j)
    -
    \frac1{N'\sigma^2}
    \sum_{\ell=1}^{N'}
    (x'_\ell-x_i) k_\sigma(x_i,x'_\ell).
\end{equation*}
The corresponding particle update follows:
\begin{equation*}
    x_i
    \leftarrow
    x_i
    +
    \eta
    \widehat V_{\widehat q_t,\widehat p}^{\mathrm{MMD}}(x_i).
\end{equation*}

\paragraph{Gradient flow of KL divergence.}
For $p,q\in \mathcal{P}(\mathbb{R}^d)$, we next consider the KL energy
\begin{equation*}
    \mathcal F_{\mathrm{KL}}(q)
    :=
    D_{\mathrm{KL}}(q\|p)
    =
    \int q(x)\log \frac{q(x)}{p(x)} dx,
\end{equation*}
assuming that $q$ and $p$ admit smooth positive densities with respect to the Lebesgue measure.
For a perturbation $q_s = q + s\,\chi$ with $\int \chi(x)dx=0$, we have
\begin{align*}
    \frac{d}{ds}\mathcal F_{\mathrm{KL}}(q_s)\bigg|_{s=0}
    &=
    \int
    \left(
        \log q(x) - \log p(x) + 1
    \right)\chi(x)dx.
\end{align*}
Therefore, up to an additive constant,
\begin{equation*}
    \frac{\delta \mathcal F_{\mathrm{KL}}}{\delta q}(q)(x)
    =
    \log q(x) - \log p(x).
\end{equation*}
The induced Wasserstein velocity field is
\begin{align*}
    V_t^{\mathrm{KL}}(x)
    &=
    -\nabla_x
    \frac{\delta \mathcal F_{\mathrm{KL}}}{\delta q}(q_t)(x) \\
    &=
    \nabla_x \log p(x)
    -
    \nabla_x \log q_t(x),
\end{align*}
with $q_t$ the corresponding flow under the initial condition $q_0$. This recovers the KL velocity field reported in Table~\ref{tab:wgf_functionals}. The first term
attracts particles toward high-density regions of the target distribution, while the second term is
a repulsive score term induced by the current model distribution.

\paragraph{Particle-based estimator for KL.}
Unlike MMD and Sinkhorn divergence, the KL velocity field requires access to the score functions
$\nabla \log p$ and $\nabla \log q_t$. In the absence of analytic scores, a particle-based estimator
can be constructed through kernel density estimation (KDE).

Let $k_\sigma(\cdot,\cdot)$ be a smooth kernel with bandwidth $\sigma>0$, for example,
\begin{equation*}
    k_\sigma(x,y)
    =
    \exp\left(-\frac{\|x-y\|^2}{2\sigma^2}\right).
\end{equation*}
Using target particles $\{y_j\}_{j=1}^M \sim p$ and an independent generated batch
$\{x'_\ell\}_{\ell=1}^{N'} \sim q_t$, define
\begin{equation*}
    \widehat p_\sigma(x)
    =
    \sum_{j=1}^M b_j k_\sigma(x,y_j),
    \qquad
    \widehat q_{t,\sigma}(x)
    =
    \sum_{\ell=1}^{N'} a'_\ell k_\sigma(x,x'_\ell).
\end{equation*}
The KDE-based score estimators are
\begin{equation*}
    \widehat s_p(x)
    :=
    \nabla_x \log \widehat p_\sigma(x)
    =
    \frac{\sum_{j=1}^M b_j \nabla_x k_\sigma(x,y_j)}
    {\sum_{j=1}^M b_j k_\sigma(x,y_j)},
\end{equation*}
and
\begin{equation*}
    \widehat s_{q_t}(x)
    :=
    \nabla_x \log \widehat q_{t,\sigma}(x)
    =
    \frac{\sum_{\ell=1}^{N'} a'_\ell \nabla_x k_\sigma(x,x'_\ell)}
    {\sum_{\ell=1}^{N'} a'_\ell k_\sigma(x,x'_\ell)}.
\end{equation*}
The particle estimator of the KL-induced velocity is
\begin{equation*}
    \widehat V_{\widehat q_t,\widehat p}^{\mathrm{KL}}(x_i)
    =
    \widehat s_p(x_i)
    -
    \widehat s_{q_t}(x_i).
\end{equation*}
For the Gaussian kernel, we have
\begin{equation*}
    \nabla_x k_\sigma(x,y)
    =
    \frac{y-x}{\sigma^2}k_\sigma(x,y).
\end{equation*}
Therefore,
\begin{equation*}
    \widehat s_p(x_i)
    =
    \frac{1}{\sigma^2}
    \left(
        \frac{\sum_{j=1}^M b_j k_\sigma(x_i,y_j)y_j}
        {\sum_{j=1}^M b_j k_\sigma(x_i,y_j)}
        -
        x_i
    \right),
\end{equation*}
and similarly
\begin{equation*}
    \widehat s_{q_t}(x_i)
    =
    \frac{1}{\sigma^2}
    \left(
        \frac{\sum_{\ell=1}^{N'} a'_\ell k_\sigma(x_i,x'_\ell)x'_\ell}
        {\sum_{\ell=1}^{N'} a'_\ell k_\sigma(x_i,x'_\ell)}
        -
        x_i
    \right).
\end{equation*}
The particle-based estimate of the velocity field simplifies to the difference between two kernel-weighted local
means:
\begin{equation*}
    \widehat V_{\widehat q_t,\widehat p}^{\mathrm{KL}}(x_i)
    =
    \frac{1}{\sigma^2}
    \left[
        \frac{\sum_{j=1}^M b_j k_\sigma(x_i,y_j)y_j}
        {\sum_{j=1}^M b_j k_\sigma(x_i,y_j)}
        -
        \frac{\sum_{\ell=1}^{N'} a'_\ell k_\sigma(x_i,x'_\ell)x'_\ell}
        {\sum_{\ell=1}^{N'} a'_\ell k_\sigma(x_i,x'_\ell)}
    \right].
\end{equation*}
The corresponding particle update then follows:
\begin{equation*}
    x_i
    \leftarrow
    x_i
    +
    \eta
    \widehat V_{\widehat q_t,\widehat p}^{\mathrm{KL}}(x_i).
\end{equation*}

\subsection{More discussion on the estimators for self-transport}\label{app:self_transport}

\begin{figure}[t]
    \centering
    \includegraphics[width=0.8\linewidth]{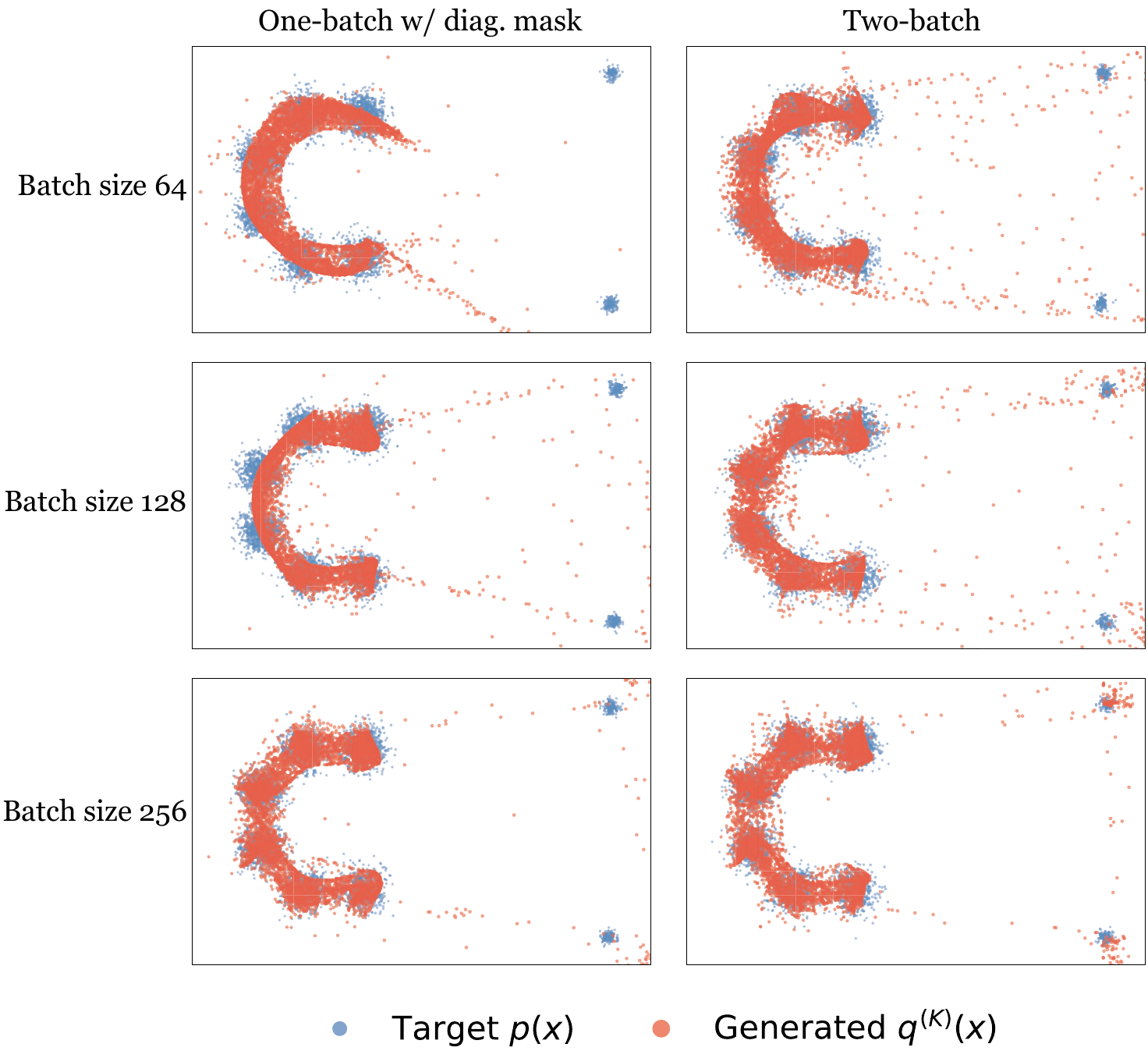}
    \caption{Evaluation of self-transport estimators on a 2D Gaussian mixture dataset featuring six dominant modes and two distant minority modes. ``One-batch'' refers to the single batch estimator with diagonal masking. ``Two-batch'' refers to our two-batch  estimator with resampling.}
    \label{fig:debiased_self_transport}
\end{figure}

As introduced in Sec.~\ref{subsec:particle}, estimating the self-interaction velocity field $T_{q,q}^\varepsilon(x)$ requires computing the OT plan of the model distribution $q$ with itself. A naive empirical approximation based on a single mini-batch of size $B$, i.e., $\{x_i\}_{i=1}^B$, inherently suffers from a self-matching bias, as shown in Fig.~\ref{fig:two_batch}. Following \cite{deng2026generative}, which addresses a similar issue, one may instead apply a diagonal mask to the cost matrix (setting $C_{ii} = \infty$), which explicitly prevents any particle from matching with itself.

However, this masking strategy introduces a new source of bias that is sensitive to the batch size $B$. For any given particle $x_i$, the diagonal mask effectively forces it to transport its mass to an empirical target distribution composed of the remaining particles:
\begin{equation*}
    \widehat{q}_{-i} = \frac{1}{B-1} \sum_{j \neq i} \delta_{x_j}.
\end{equation*}
Because $x_i$ is excluded from its own target space, the particle is artificially forced to move a strictly positive distance towards its neighbors, even if it is already situated perfectly at a high-density peak of the true distribution $q$. The magnitude of this structural bias scales proportionally to $\mathcal{O}(1/B)$, effectively meaning that a fraction of the mass corresponding to roughly $(B-1)/B$ is mapped sub-optimally. As $B$ becomes smaller, this forced redistribution dominates the velocity field, pushing particles away from their true local density centers (as observed in the left column of Fig.~\ref{fig:debiased_self_transport}).

By drawing an independent second batch $\{x_l'\}_{l=1}^B \sim q_t$, our two-batch estimator effectively avoids this issue. The OT plan is computed between two independent empirical measures, allowing a particle $x_i$ to match freely with a particle $x_l'$ that happens to be infinitesimally close to it, without any artificial $\infty$-cost penalty. This yields a consistent estimate of the continuous self-transport map that remains robust even at small batch sizes.

\subsection{More discussions on  velocity guidance}
\label{app:kl_velocity_guidance}

In this section, we first show that our proposed velocity guidance recovers the standard
exponentially tilted target distribution when the energy functional is chosen as the KL
divergence. We then show how this understanding is leveraged to construct the proposed velocity guidance in Eq.~\eqref{eq:velocity-cfg}, which also serves as a conceptual justification for Eq.~\eqref{eq:velocity-cfg}.

Let $p(\cdot|c)$ denote the conditional data distribution, $p(\cdot|\varnothing)$ 
denote the unconditional data distribution, and $w \ge 0$ be the guidance weight. 
In analogy with Eq.~\eqref{eq:velocity-cfg}, we construct the velocity-guidance field by adding a term scaled by $w$ that points from the unconditional to the conditional distribution:
\begin{equation}
    \widetilde V_{t}^{\mathrm{KL},w}(x)
    :=
    \underbrace{\Big(
        \nabla_x \log p(x|c)
        -
        \nabla_x \log q_t(x|c)
    \Big)}_{V^{\mathrm{KL}}_{q_t(\cdot|c),p(\cdot|c)}}
    +
    w
    \underbrace{\Big(
        \nabla_x \log p(x|c)
        -
        \nabla_x \log p(x|\varnothing)
    \Big)}_{V^{\mathrm{KL}}_{q_t(\cdot|c),p(\cdot|c)} - V^{\mathrm{KL}}_{q_t(\cdot|c),p(\cdot|\varnothing)}}.
    \label{eq:kl_velocity_guidance}
\end{equation}

We formalize the connection between this constructed velocity field and CFG in the following proposition.

\begin{Proposition}
    \label{prop:kl_cfg}
    The velocity-guidance field $\widetilde V_{t}^{\mathrm{KL},w}$ in Eq.~\eqref{eq:kl_velocity_guidance} is exactly the WGF velocity field of the combined energy functional
    \begin{equation*}
        \mathcal F_{\mathrm{KL}}^w(q_t(\cdot|c))
        =
        (1+w) D_{\mathrm{KL}}\big(q_t(\cdot|c)\|p(\cdot|c)\big)
        -
        w D_{\mathrm{KL}}\big(q_t(\cdot|c)\|p(\cdot|\varnothing)\big).
    \end{equation*}
    Furthermore, the unique global minimizer of $\mathcal F_w$ is the exponentially tilted target distribution $p_w(\cdot|c)$, defined as
    \begin{equation*}
        p_w(x|c)
        :=
        \frac{1}{Z_w(c)}
        p(x|c)
        \left(
            \frac{p(x|c)}{p(x|\varnothing)}
        \right)^w,
    \end{equation*}
    where $Z_w(c)$ is the normalizing constant.
\end{Proposition}

\begin{proof}
    We first expand the combined energy functional $\mathcal F_{\mathrm{KL}}^w(q_t(\cdot|c))$:
    \begin{align*}
        \mathcal F_{\mathrm{KL}}^w(q_t)
        &=
        (1+w) \int q_t \log \frac{q_t}{p(\cdot|c)}
        -
        w \int q_t \log \frac{q_t}{p(\cdot|\varnothing)} \\
        &=
        \int q_t \log q_t
        -
        \int q_t \log \Big( p(\cdot|c)^{1+w} p(\cdot|\varnothing)^{-w} \Big) \\
        &=
        \int q_t \log \frac{q_t}{p_w(\cdot|c)}
        -
        \log Z_w(c) \\
        &=
        D_{\mathrm{KL}}\big(q_t(\cdot|c)\|p_w(\cdot|c)\big)
        -
        \log Z_w(c).
    \end{align*}
    Since $\log Z_w(c)$ is a constant with respect to the distribution $q_t$, it is evident that $\mathcal F_{\mathrm{KL}}^w(q_t(\cdot|c))$ achieves its unique global minimum when $q_t(\cdot|c) = p_w(\cdot|c)$, which proves the second part of the proposition.

    To derive the induced WGF, we compute the first variation of $\mathcal F_{\mathrm{KL}}^w(q_t(\cdot|c))$. Because the additive constant $-\log Z_w(c)$ vanishes upon differentiation, the first variation is identical to that of $D_{\mathrm{KL}}\big(q_t(\cdot|c)\|p_w(\cdot|c)\big)$:
    \begin{equation*}
        \frac{\delta \mathcal F_{\mathrm{KL}}^w}{\delta q}(q_t(\cdot|c))(x)
        =
        \log q_t(x|c)
        -
        \log p_w(x|c)
        +
        \mathrm{const.}
    \end{equation*}
    The corresponding Wasserstein velocity field is the negative spatial gradient of the first variation:
    \begin{align*}
        -\nabla_x \frac{\delta \mathcal F_{\mathrm{KL}}^w}{\delta q}(q_t(\cdot|c))(x)
        &=
        \nabla_x \log p_w(x|c)
        -
        \nabla_x \log q_t(x|c) \\
        &=
        \nabla_x \Big( (1+w)\log p(x|c) - w\log p(x|\varnothing) - \log Z_w(c) \Big)
        -
        \nabla_x \log q_t(x|c) \\
        &=
        (1+w)\nabla_x \log p(x|c)
        -
        w\nabla_x \log p(x|\varnothing)
        -
        \nabla_x \log q_t(x|c).
    \end{align*}
    Rearranging the terms, we obtain:
    \begin{equation*}
        -\nabla_x \frac{\delta \mathcal F_{\mathrm{KL}}^w}{\delta q}(q_t(\cdot|c))(x)
        =
        \Big( \nabla_x \log p(x|c) - \nabla_x \log q_t(x|c) \Big)
        +
        w \Big( \nabla_x \log p(x|c) - \nabla_x \log p(x|\varnothing) \Big),
    \end{equation*}
    which is exactly the constructed velocity-guidance field $\widetilde V_{t}^{\mathrm{KL},w}(x)$ defined in Eq.~\eqref{eq:kl_velocity_guidance}.
\end{proof}

Therefore, under the KL energy, our velocity-guidance formulation is precisely equivalent to running the WGF toward the exponentially tilted distribution $p_w(x|c)$. This corresponds exactly to the standard CFG target form, where the conditional distribution is amplified relative to the unconditional distribution by the density ratio $p(x|c)/p(x|\varnothing)$. The results in Appendix~\ref{app:additional-ablation} verify that velocity guidance yields enhanced performance over distribution guidance when KL divergence is leveraged as the energy functional.

\begin{figure}[t!]
    \centering
    \includegraphics[width=1.0\linewidth]{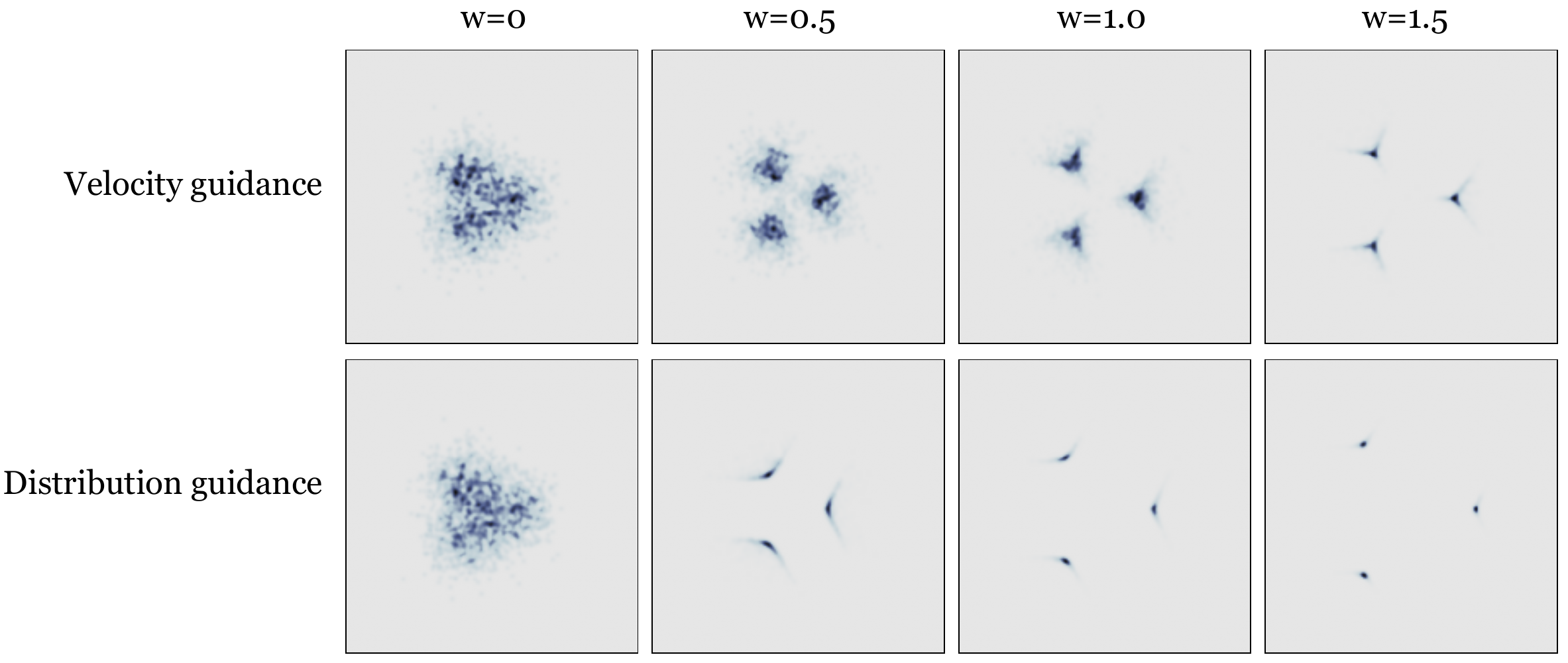}
    \caption{Comparison of velocity guidance and distribution guidance for conditional generation on a 2D 3-mode Gaussian mixture dataset.}
    \label{fig:cfg_compare}
\end{figure}

\paragraph{Connection to the Sinkhorn velocity guidance.}
For the Sinkhorn divergence, the same closed-form target distribution is generally unavailable,
because the velocity field is expressed through barycentric projections rather than score functions.
Nevertheless, the KL case shows that adding the velocity-level correction
\begin{equation*}
    V_{q_t(\cdot|c),p(\cdot|c)} - V_{q_t(\cdot|c),p(\cdot|\varnothing)}
\end{equation*}
has the correct interpretation in the score-based limit: it changes the effective target from
$p(x|c)$ to the exponentially tilted distribution
$p(x|c)(p(x|c)/p(x|\varnothing))^w$. This motivates applying the analogous guidance design to the Sinkhorn-induced velocity field. Specifically, we have
\begin{align}
\nonumber
\tilde{V}^{\varepsilon,w}_{q_\theta,p}(x_i)\!=\! \underbrace{\left(T^\varepsilon_{q_\theta(\cdot|c),p(\cdot|c)}(x_i) \!-\! T^\varepsilon_{q_\theta(\cdot|c),q_\theta(\cdot|c)}(x_i)\right)}_{V^\epsilon_{q_\theta(\cdot|c),p(\cdot|c)}} \!+\! w \underbrace{\left(T^\varepsilon_{q_\theta(\cdot|c),p(\cdot|c)}(x_i) \!-\! T^\varepsilon_{q_\theta(\cdot|c),p(\cdot|\varnothing)}(x_i)\right)}_{V^\epsilon_{q_\theta(\cdot|c),p(\cdot|c)} - V^\epsilon_{q_\theta(\cdot|c),p(\cdot|\varnothing)}},
\end{align}
where 
\begin{align}
V^\epsilon_{q_\theta(\cdot|c),p(\cdot|c)}(x_i)&=T^\varepsilon_{q_\theta(\cdot|c),p(\cdot|c)}(x_i) \!-\! T^\varepsilon_{q_\theta(\cdot|c),q_\theta(\cdot|c)}(x_i),\\
V^\epsilon_{q_\theta(\cdot|c),p(\cdot|\varnothing)}(x_i)&=T^\varepsilon_{q_\theta(\cdot|c),p(\cdot|\varnothing)}(x_i) \!-\! T^\varepsilon_{q_\theta(\cdot|c),q_\theta(\cdot|c)}(x_i),
\end{align}
which recovers Eq.~\eqref{eq:velocity-cfg}.

\textbf{Toy illustration.} To empirically compare our proposed velocity guidance with distribution guidance, we evaluate conditional generation on a 2D 3-mode Gaussian mixtures. Here, each mode is treated as a distinct class. We train a generator $x=f_\theta(z,c,w)$ where the guidance scale $w$ is randomly sampled during training. 

As illustrated in Fig.~\ref{fig:cfg_compare} (top row), velocity guidance reflects the behavior of an exponentially tilted target distribution $p_w(x|c)$. As $w$ increases, the probability mass is amplified at the highest-density regions of the conditional mode. In contrast, the distribution guidance implicitly targets a linear mixture between the conditional and unconditional distributions. Fig.~\ref{fig:cfg_compare} (bottom row) demonstrates that increasing $w$ causes particles to scatter away from the modes.

\subsection{More discussions on related work}\label{subsec:discussion}
Our proposed~\method framework is related to several prominent directions in generative modeling. We highlight the key distinctions below.

\begin{figure}[t!]
    \centering
    \includegraphics[width=0.9\linewidth]{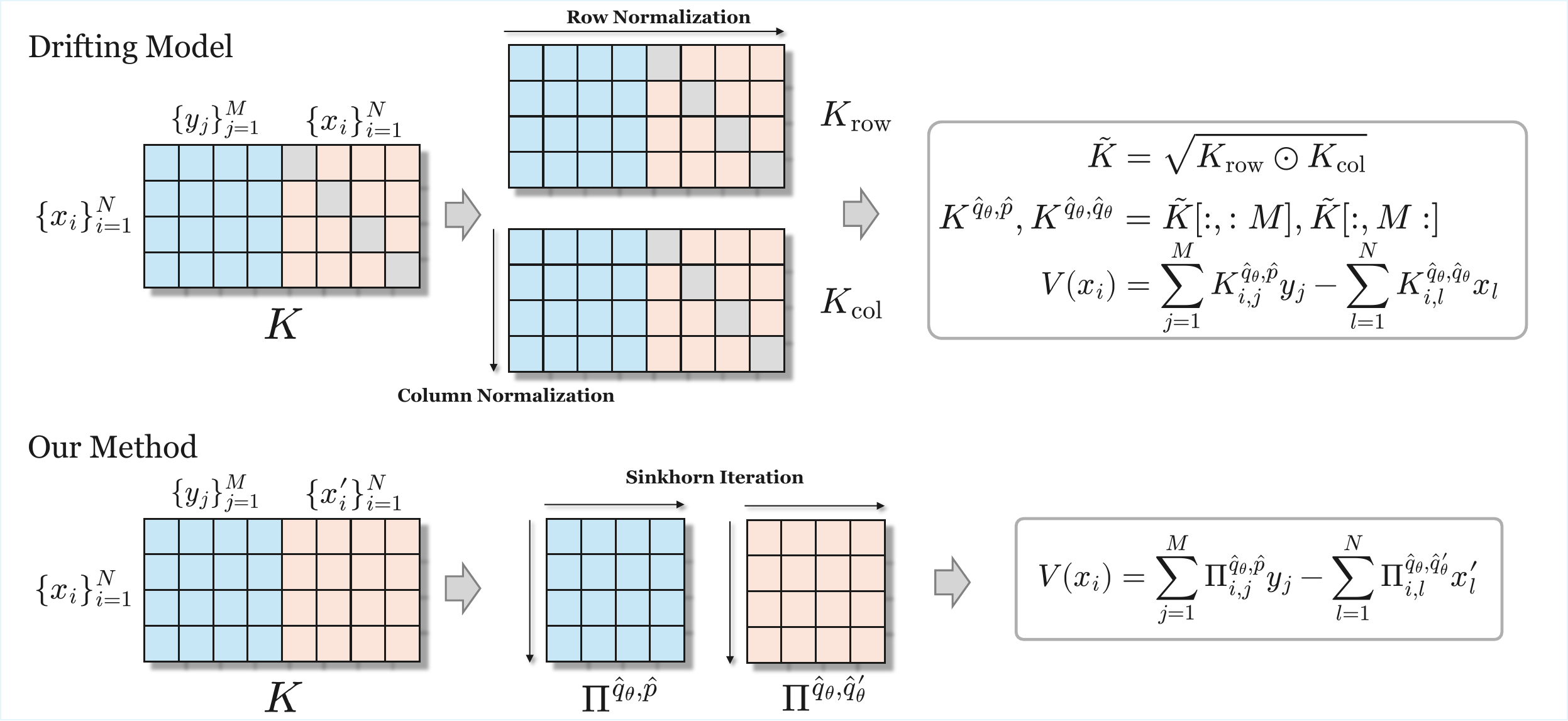}
    \caption{Illustrations of the difference in the velocity field computation between Drifting Model (Alg. 2 in~\cite{deng2026generative}) and our approach. ``$\sqrt{\cdot}$'' refers to the element-wise square root, and the \textcolor{gray}{gray blocks} refer to the diagonal mask adopted in~\cite{deng2026generative}. Drifting Model adopts a heuristic construction with separate row-wise and column-wise Softmax normalizations. Our approach leverages Sinkhorn iteration to compute the barycentric projection.}
    \label{fig:compute-illustration}
\end{figure}

\textbf{Connection to Drifting Model.}
Our formulation shares the forward, iterative pushforward structure with Drifting Model~\cite{deng2026generative}. The main difference lies in the construction of the velocity field. Drifting Model employs heuristic attraction and repulsion terms ($V_p^{+}$ and $V_q^{-}$) based on localized kernel interactions, resulting in a mean-shift-type dynamics driven by local proximity. Such local updates can lead to unstable behavior and require careful tuning to avoid mode collapse.

In contrast,~\method derives the velocity field from a WGF and instantiates it via Sinkhorn divergence. The resulting dynamics are governed by global OT couplings, which enforce marginal constraints and coordinate particle movement at the distribution level. This yields a principled and globally consistent transport mechanism.

\textbf{Distinction from JKO-based WGF models.}
Another line of work formulates generative processes as WGFs using the JKO scheme~\cite{choi2024scalable}, which corresponds to an implicit time discretization. However, each JKO step requires solving a nontrivial optimization problem, which in practice often necessitates adversarial training, thereby introducing training instability and limiting scalability.

Distinct from~\cite{choi2024scalable},~\method adopts an explicit Euler discretization of the flow. By computing the Sinkhorn velocity field $V^\varepsilon$ via efficient Sinkhorn scaling, the learning problem reduces to a simple regression objective, avoiding inner optimization loops and eliminating the need for adversarial training. This enables scalable learning of a one-step pushforward map.

\textbf{Distinction from Sinkhorn GANs.}
The use of EOT in generative modeling was popularized by Sinkhorn GAN~\cite{genevay2018learning}, which directly employs the Sinkhorn divergence as the target loss function. This approach requires auto-differentiating through the Sinkhorn iterations and backpropagating gradients through the transport plan, inducing numerical instability and memory overhead.

\method instead interprets the Sinkhorn divergence as an energy functional that defines the velocity field of the dynamics. The transport plan is treated as an empirical oracle, without backpropagation through Sinkhorn iterations. This avoids instability and allows the method to scale effectively.

\textbf{Distinction from moment matching.}
Generative Moment Matching Networks (GMMNs)~\cite{li2015generative} learn generators through particle updates driven by the squared MMD, which can be viewed as an instance of the MMD-based variant of~\method in Table~\ref{tab:wgf_functionals}. However, MMD relies on kernel-based interactions between particles, which can suffer from the curse of dimensionality and yield vanishing gradients when the current pushforward distribution is far from the target, as also reflected in our ablation experiments. Inductive Moment Matching~\cite{zhou2025inductive} performs kernel-based moment matching along a prescribed static interpolation between the prior and target distributions. As a result, it inherits the same kernel-related limitations and typically requires multiple pushforward sampling steps to simulate transitions along the interpolator.~\method instead leverages the OT-based dynamics and directly parametrizes a one-step map that does not require any intermediate simulations at inference time.

\section{Proofs}\label{app:proof}
In this section, we present the detailed statements of assumptions and theorems discussed in the main paper, along with their proofs.

In the main text, Theorem~\ref{thm:convergence} is stated for the Sinkhorn velocity
\(V_{q,p}^{\varepsilon}\) defined in \eqref{eq:continuous-sinkhorn-drift}. For generality, here we prove the result under Assumption~\ref{ass} for a generic velocity field
\[
V:\mathcal P_2(\mathbb R^d)\times \mathcal P_2(\mathbb R^d)\times \mathbb R^d
\to \mathbb R^d,
\qquad
(q,p,x)\mapsto V_{q,p}(x),
\]
where \(q\) is the evolving source distribution and \(p\) is the fixed target distribution. Proposition~\ref{prop:sinkhorn-bounded-domain} then verifies the assumption for the Sinkhorn velocity when \(q_0\) and \(p\) have bounded support.

\subsection{Complete statement and the proof of Theorem~\ref{thm:convergence}}

For \(N\in\mathbb{N}\) and step size \(\eta>0\), let particles \(\{x_i^{(k)}\}_{i=1}^N\) evolve according to
\begin{equation}
\label{eq:euler_particles}
x_i^{(k+1)}=x_i^{(k)}+\eta\,V_{\widehat q_k^N,\widehat p^M}(x_i^{(k)}),
\qquad
\widehat q_k^{N}:=\frac1N\sum_{i=1}^N\delta_{x_i^{(k)}},
\qquad
t_k:=k\eta,
\end{equation}
with the initial data points $x_i^{(0)}$  sampled i.i.d. from $q_0$.

Define the piecewise-linear interpolation \(x_i^{\eta}:[0,\infty)\to\mathbb{R}^d\) by
\begin{equation}\label{eq:linear_interpolation}
x_i^{\eta}(t):=x_i^{(k)}+(t-t_k)\,V_{\widehat q_k^{N},\widehat p^M}(x_i^{(k)}),\quad t\in[t_k,t_{k+1}), \quad i=1,2,\cdots,N,
\end{equation}
and the associated empirical measure
\[
\widehat q_t^{N,M,\eta}:=\frac1N\sum_{i=1}^N\delta_{x_i^{\eta}(t)}.
\]

We study the limiting behavior of \eqref{eq:linear_interpolation} under a joint asymptotic regime in which \(N,M\to\infty\) and \(\eta\to 0\).

Let $m_2(\mu)$ be the root second moment of $\mu\in \mathcal{P}_2(\mathbb{R}^d)$:
\[
m_2(\mu):=\left(\int_{\mathbb R^d}\|x\|^2\,d\mu(x)\right)^{1/2}.
\]

\begin{Assumption}\label{ass}
Assume that \(V_{q,p}\) is Borel measurable and belongs to \(L^2(q)\) for every
\(q,p\in\mathcal P_2(\mathbb R^d)\). Assume that there exist constants
\(L_0,L_1,L_2,L_3<\infty\) such that, for all
\(q,\widetilde q,p,\widetilde p\in\mathcal P_2(\mathbb R^d)\),
\begin{eqnarray}
    \|V_{q,p}\|_{L^2(q)}
&\le& L_0\bigl(m_2(q)+m_2(p)\bigr),
\label{ass:L2growth}\\
\|V_{q,p}(x)-V_{q,p}(x')\|
&\le& L_1\|x-x'\|,
\qquad x,x'\in\mathbb R^d,
\label{ass:spatialLip}\\
\|V_{q,p}-V_{\widetilde q,p}\|_{L^2(\widetilde q)}
&\le& L_2\mathcal W_2(q,\widetilde q),
\label{ass:L2q}\\
\|V_{q,p}-V_{q,\widetilde p}\|_{L^2(q)}
&\le& L_3\mathcal W_2(p,\widetilde p).
\label{ass:L2p}
\end{eqnarray}
\end{Assumption}

\begin{Remark}[Some examples on \(\mathcal P_2(\mathbb R^d)\)]
    The conditions in Assumption~\ref{ass} are in the spirit of the Lipschitz and
linear-growth assumptions commonly used for stability of continuity equations
and mean-field particle systems; see, e.g.,
\cite{ambrosio2008gradient,carmona2018probabilistic}.   Assumption~\ref{ass} can be verified directly for several natural velocity
fields on the full space \(\mathbb R^d\). One example is the velocity induced
by the squared MMD energy. More precisely, for a positive-definite kernel
\(k:\mathbb R^d\times\mathbb R^d\to\mathbb R\), recall the corresponding energy potential defined in \eqref{eq:MMD_energy}.
If \(k(x,y)=\kappa(x-y)\) is translation invariant and
\(\nabla \kappa:\mathbb R^d\to\mathbb R^d\) is globally Lipschitz, then the
corresponding Wasserstein velocity becomes
\[
    V_{q,p}(x)
    =
    \int_{\mathbb R^d} \kappa(x-y)\,p(dy)
    -
    \int_{\mathbb R^d} \kappa(x-y)\,q(dy).
\]
In this case, Assumption~\ref{ass} holds on \(\mathcal P_2(\mathbb R^d)\), with
constants depending only on the Lipschitz constant of \(\nabla \kappa\). For example, the Gaussian kernel $\exp\left(-\frac{\|x-y\|^2}{2\sigma^2}\right)$ belongs to this class.
Another simple example is the moment-matching velocity
\[
    V_{q,p}(x)
    =
    \bar p-\bar q,
  \,\, \textrm{ with }\,\,
    \bar q:=\int_{\mathbb R^d} y\,q(dy),
\,\,
    \bar p:=\int_{\mathbb R^d} y\,p(dy),
\]
which also satisfies Assumption~\ref{ass}  on
\(\mathcal P_2(\mathbb R^d)\). 
\end{Remark}

\begin{Remark}[Localization on bounded domains]\label{rem:bounded-localization}
When the initial distribution and the target distribution have bounded supports, the Sinkhorn
velocity satisfies the conditions in Assumption~\ref{ass} along the relevant flow trajectory.
Indeed, the Sinkhorn flow and its empirical approximations remain supported on a bounded domain
over every finite time interval \([0,T]\); see Proposition~\ref{prop:sinkhorn-bounded-domain}.
Therefore, it is enough to verify the estimates in Assumption~\ref{ass} on this bounded class of
measures, rather than globally on \(\mathcal P_2(\mathbb R^d)\).
\end{Remark}

Under these conditions, one can establish a quantitative convergence result showing that the interpolated empirical measure \(\widehat q_t^{N,M,\eta}\) converges to the unique solution of a nonlinear PDE as \(N,M\to\infty\) and \(\eta\to0\).

Now we present the \emph{formal version} of Theorem~\ref{thm:convergence} below, followed by the detailed proof.
\begin{Theorem}\label{thm:convergence_detail} Assume Assumption \ref{ass} holds and $q_0, p\in \mathcal P_2(\mathbb{R}^d)$.
   For every $T>0$, there exists a unique
$q\in C([0,T];\mathcal P_2(\mathbb R^d))$ solving the nonlinear continuity equation
\begin{equation}
\label{eq:mf_pde}
\partial_t q_t+\nabla\cdot\big(q_t\,V_{q_t,p} \big)=0
\end{equation}
in the weak sense
\begin{equation*}
\label{eq:weak_form}
\int_{\mathbb R^d}\varphi\,dq_t-\int_{\mathbb R^d}\varphi\,dq_0
=
\int_0^t\int_{\mathbb R^d}\nabla\varphi(x)\cdot V_{q_s,p} (x)\,dq_s(x)\,ds,
\qquad \forall \varphi\in C_c^1(\mathbb R^d).
\end{equation*}
Moreover, for all $N, M$ and all $\eta>0$,
\begin{equation*}
\label{eq:quant_bound_relax}
\begin{aligned}
    \sup_{t\in[0,T]} \mathcal W_2\big(\widehat q_t^{N,M,\eta},q_t\big)
    &\le
e^{(L_1+L_2)T} \mathcal W_2(\widehat q_0^N,q_0)\\
&\quad + \frac{L_3}{L_1+L_2}\bigl(e^{(L_1+L_2)T}-1\bigr)
\mathcal W_2(\widehat p^M,p)\\
&\quad + C_T\,\eta\bigl[m_2(\widehat q_0^N)+m_2(\widehat p^M)\bigr],
\end{aligned}
\end{equation*}
where \(C_T<\infty\) depends only on \(T,L_0,L_1,L_2\).
Consequently,  as $\eta\rightarrow 0$ and $N, M\to\infty$,
\begin{equation}\label{eq:bound_uniform}
    \sup_{t\in[0,T]} \mathcal W_2\big(\widehat q_t^{N,M,\eta},q_t\big)\to0, \qquad\text{almost surely.}
\end{equation}
\end{Theorem}

\begin{proof}
For generic \(\pmb{x}=(x_1,\ldots,x_N)\in(\mathbb R^d)^N\), define
\begin{eqnarray*}
    \widehat q^N:=\frac1N\sum_{i=1}^N\delta_{x_i},
\quad
V^{N,M}(\pmb{x})
:=
\bigl(
V_{\widehat q^N,\widehat p^M} (x_1),\ldots,
V_{\widehat q^N,\widehat p^M} (x_N)
\bigr).
\end{eqnarray*}
Equip \((\mathbb R^d)^N\) with the norm
\[
\|\pmb{x}\|_{2,N}:=\left(\frac1N\sum_{i=1}^N \|x_i\|^2\right)^{1/2}.
\]
Let \(\pmb{x}'=(x'_1,\ldots,x'_N)\), and set
\(\widehat q^{N\prime}:=N^{-1}\sum_{i=1}^N\delta_{x'_i}\). By
\eqref{ass:spatialLip} and \eqref{ass:L2q},
\begin{align*}
\|V^{N,M}(\pmb{x})-V^{N,M}(\pmb{x}')\|_{2,N}
&\le
\left(\frac1N\sum_{i=1}^N
\bigl\|V_{\widehat q^N,\widehat p^M} (x_i)
      -V_{\widehat q^N,\widehat p^M} (x'_i)\bigr\|^2
\right)^{1/2} \\
&\quad+
\left(\frac1N\sum_{i=1}^N
\bigl\|V_{\widehat q^N,\widehat p^M} (x'_i)
      -V_{\widehat q^{N\prime},\widehat p^M} (x'_i)\big\|^2
\right)^{1/2} \\
&\le L_1\|\pmb{x}-\pmb{x}'\|_{2,N}
+L_2\mathcal W_2(\widehat q^N,\widehat q^{N\prime}) \\
&\le (L_1+L_2)\|\pmb{x}-\pmb{x}'\|_{2,N},
\end{align*}
where 
\[
\mathcal W_2(\widehat q^N,\widehat q^{N\prime})
\le \left(\frac1N\sum_{j=1}^N\|x_j-x_j'\|^2\right)^{1/2}
=\|\pmb{x}-\pmb{x}'\|_{2,N}
\]
by coupling  \(x_i\) with \(x'_i\). Moreover, by \eqref{ass:L2growth},
\[
\|V^{N,M}(\pmb{x})\|_{2,N}
=
\|V_{\widehat q^N,\widehat p^M} \|_{L^2(\widehat q^N)}
\le L_0\bigl(m_2(\widehat q^N)+m_2(\widehat p^M)\bigr).
\]
Thus the finite-particle vector field is globally Lipschitz and has linear growth in
\(\|\cdot\|_{2,N}\).

Consider now the ODE
\begin{equation}
\label{eq:odeN}
\dot{\pmb{x}}^{N,M}(t)=V^{N,M}(\pmb{x}^{N,M}(t)),
\qquad
\pmb{x}^{N,M}(0)=(x^{(0)}_1,\dots,x^{(0)}_N),
\end{equation}
with the initialization \(x^{(0)}_i\) defined below \eqref{eq:euler_particles}.
Since \(V^{N,M}\) is globally Lipschitz in \(\|\cdot\|_{2,N}\), it holds that \eqref{eq:odeN} has a unique global solution by the Picard--Lindel\"of theorem.

Let
\[
\widehat{ q}_t^{N,M}:=\frac1N\sum_{i=1}^N\delta_{x_i^{N,M}(t)}.
\]
For any \(\varphi\in C_c^1(\mathbb R^d)\),
\begin{eqnarray*}
    \frac{d}{dt}\int \varphi\,d\widehat  q_t^{N,M}
&=&
\frac{d}{dt}\frac1N\sum_{i=1}^N \varphi(x_i^{N,M}(t))
=
\frac1N\sum_{i=1}^N \nabla\varphi(x_i^{N,M}(t))\cdot \dot x_i^{N,M}(t)\\
&=&
\int \nabla\varphi(x)\cdot V_{ \widehat q_t^{N,M},\widehat p^M} (x)\,d\widehat q_t^{N,M}(x).
\end{eqnarray*}
Thus \(\widehat q_t^{N,M}\) is a weak solution of
\[
\partial_t \widehat{ q}_t^{N,M}+\nabla\cdot\left(\widehat{ q}_t^{N,M}V_{\widehat{q}_t^{N,M},\widehat p^M} \right)=0.
\]

We next derive the stability estimate. Let
\(\mu,\nu\in C([0,T];\mathcal P_2(\mathbb R^d))\) be weak solutions of
\[
\partial_t \mu_t+\nabla\cdot(\mu_t V_{\mu_t,p} )=0,
\qquad
\partial_t \nu_t+\nabla\cdot(\nu_t V_{\nu_t,\widetilde p} )=0,
\]
where \(p,\widetilde p\in\mathcal P_2(\mathbb R^d)\) are fixed. First note that
\eqref{ass:L2growth} and \eqref{ass:spatialLip} imply the pointwise linear-growth bound
\begin{equation}
\label{eq:pointwise-growth-from-L2}
\|V_{r,a} (x)\|
\le
L_1\|x\|+(L_0+L_1)m_2(r)+L_0m_2(a),
\qquad r,a\in\mathcal P_2(\mathbb R^d).
\end{equation}
Indeed,
\[
\|V_{r,a} (0)\|
\le
\left(\int \|V_{r,a} (0)-V_{r,a} (z)\|^2\,dr(z)\right)^{1/2}
+
\|V_{r,a} \|_{L^2(r)}
\le
L_1m_2(r)+L_0\bigl(m_2(r)+m_2(a)\bigr).
\]
Since \(\mu,\nu\in C([0,T];\mathcal P_2(\mathbb R^d))\), for these fixed
solution trajectories we have
\[
M_T:=\sup_{t\in[0,T]}\bigl(m_2(\mu_t)+m_2(\nu_t)\bigr)<\infty .
\]
Let \(\gamma_0\in\Pi(\mu_0,\nu_0)\) be optimal for
\(\mathcal W_2(\mu_0,\nu_0)\). For each \(x_0,y_0\in\mathbb R^d\), define
\(X(\cdot;x_0)\) and \(Y(\cdot;y_0)\) as the solutions on \([0,T]\) of
\[
\begin{cases}
\dfrac{d}{dt}X(t;x_0)=V_{\mu_t,p}(X(t;x_0)),\\
X(0;x_0)=x_0,
\end{cases}
\qquad
\begin{cases}
\dfrac{d}{dt}Y(t;y_0)=V_{\nu_t,\widetilde p}(Y(t;y_0)),\\
Y(0;y_0)=y_0.
\end{cases}
\]
The global well-posedness of these ODEs on \([0,T]\) follows from
\eqref{eq:pointwise-growth-from-L2} and \eqref{ass:spatialLip}, since
\(\mu,\nu\in C([0,T];\mathcal P_2(\mathbb R^d))\). Since \(\mu_t\) and
\(\nu_t\) solve the corresponding continuity equations, their characteristic
representations give
\[
X(t;\cdot)_\#\mu_0=\mu_t,
\qquad
Y(t;\cdot)_\#\nu_0=\nu_t.
\]
Therefore
\[
\bigl(X(t;\cdot),Y(t;\cdot)\bigr)_\#\gamma_0\in\Pi(\mu_t,\nu_t).
\]
Define
\[
D(t):=
\left(
\int_{\mathbb R^d\times\mathbb R^d}
\|X(t;x_0)-Y(t;y_0)\|^2\,d\gamma_0(x_0,y_0)
\right)^{1/2}.
\]
Since \(\bigl(X(t;\cdot),Y(t;\cdot)\bigr)_\#\gamma_0\in\Pi(\mu_t,\nu_t)\), we have
\[
\mathcal W_2(\mu_t,\nu_t)\le D(t).
\]

For \(\delta>0\), set
\[
D_\delta(t):=(D(t)^2+\delta)^{1/2}.
\]
For a.e. \(t\),
\begin{align*}
\frac{d}{dt}D_\delta(t)
&=
\frac{1}{D_\delta(t)}
\int
\bigl(X(t;x_0)-Y(t;y_0)\bigr)\cdot
\Bigl[
V_{\mu_t,p} (X(t;x_0))
-
V_{\nu_t,\widetilde p} (Y(t;y_0))
\Bigr]
\,d\gamma_0(x_0,y_0) \\
&\le
\left(
\int
\Bigl\|
V_{\mu_t,p} (X(t;x_0))
-
V_{\nu_t,\widetilde p} (Y(t;y_0))
\Bigr\|^2
\,d\gamma_0(x_0,y_0)
\right)^{1/2} \\
&\le
\left(
\int
\Bigl\|
V_{\mu_t,p} (X(t;x_0))
-
V_{\mu_t,p} (Y(t;y_0))
\Bigl\|^2
\,d\gamma_0
\right)^{1/2} \\
&\quad+
\left(
\int
\Bigl\|
V_{\mu_t,p} (Y(t;y_0))
-
V_{\nu_t,p} (Y(t;y_0))
\Bigr\|^2
\,d\gamma_0
\right)^{1/2} \\
&\quad+
\left(
\int
\Bigl\|
V_{\nu_t,p} (Y(t;y_0))
-
V_{\nu_t,\widetilde p} (Y(t;y_0))
\Bigr\|^2
\,d\gamma_0
\right)^{1/2} \\
&\le
L_1D(t)
+
\|V_{\mu_t,p} -V_{\nu_t,p} \|_{L^2(\nu_t)}
+
\|V_{\nu_t,p} -V_{\nu_t,\widetilde p} \|_{L^2(\nu_t)} \\
&\le
(L_1+L_2)D(t)+L_3\mathcal W_2(p,\widetilde p).
\end{align*}
Here we used that the law of \(Y(t;y_0)\) under \(\gamma_0\) is \(\nu_t\), together with
\(\mathcal W_2(\mu_t,\nu_t)\le D(t)\). Integrating the preceding inequality and letting
\(\delta\downarrow0\) gives
\[
D(t)
\le
D(0)+\int_0^t
\Bigl((L_1+L_2)D(s)+L_3\mathcal W_2(p,\widetilde p)\Bigr)\,ds.
\]
Since \(D(0)=\mathcal W_2(\mu_0,\nu_0)\), Gronwall's inequality yields
\begin{equation}
\label{eq:dobrushin}
\mathcal W_2(\mu_t,\nu_t)
\le
 e^{(L_1+L_2)t}\mathcal W_2(\mu_0,\nu_0)
+
\frac{L_3}{L_1+L_2}
\bigl(e^{(L_1+L_2)t}-1\bigr)
\mathcal W_2(p,\widetilde p).
\end{equation}
In particular, taking \(\widetilde p=p\) gives uniqueness.

We now prove existence. Let \((\widehat q_0^N)_N\) and \((\widehat p^M)_M\) be the empirical measures  so that
\[
\widehat q_0^N\to q_0,
\qquad
\widehat p^M\to p
\qquad\text{in }\mathcal W_2
\]
almost surely. Applying \eqref{eq:dobrushin} to \(\widehat{ q_t}^{N,M}\) and \(\widehat{ q_t}^{N',M'}\) yields
\begin{align*}
\sup_{t\in[0,T]}\mathcal W_2(\widehat q_t^{N,M},\widehat q_t^{N',M'})
&\le
e^{(L_1+L_2)T}\mathcal W_2(\widehat q_0^N,\widehat q_0^{N'})+
\frac{L_3}{L_1+L_2}\bigl(e^{(L_1+L_2)T}-1\bigr)
\mathcal W_2(\widehat p^M,\widehat p^{M'}).
\end{align*}
Since \(\widehat q_0^N\to q_0\) and \(\widehat p^M\to p\) in \(\mathcal W_2\) almost surely, the right-hand side tends to \(0\) as
\(N,N',M,M'\to\infty\). Hence \((\widehat q^{N,M})_{N,M}\) is a Cauchy family in
\(C([0,T];\mathcal P_2(\mathbb R^d))\), where \(\mathcal P_2(\mathbb R^d)\) is equipped with \(\mathcal W_2\). Since
\((\mathcal P_2(\mathbb R^d),\mathcal W_2)\) is complete, there exists
\[
\overline q\in C([0,T];\mathcal P_2(\mathbb R^d))
\]
such that
\begin{equation}\label{eq:conv}
    \sup_{t\in[0,T]}\mathcal W_2(\widehat{q_t}^{N,M},\overline q_t)\to0 \qquad\text{almost surely} \qquad\text{as }N,M\to\infty.
\end{equation}
In particular, \(\overline q_0=q_0\).

We also need a uniform second-moment bound. Set
\[
m_2^{N,M}(t):=\left(\int \|x\|^2\,d\widehat q_t^{N,M}(x)\right)^{1/2}.
\]
For a.e. \(t\), using the ODE and Assumption~\ref{ass},
\[
\frac{d}{dt}m_2^{N,M}(t)
\le
\|V_{\widehat q_t^{N,M},\widehat p^M} \|_{L^2(\widehat q_t^{N,M})}
\le
L_0\bigl(m_2^{N,M}(t)+m_2(\widehat p^M)\bigr).
\]
Therefore, by Gronwall's inequality,
\begin{equation}
\label{eq:second-moment-bound}
m_2^{N,M}(t)
\le
e^{L_0t}m_2(\widehat q_0^N)
+
\bigl(e^{L_0t}-1\bigr)m_2(\widehat p^M),
\qquad 0\le t\le T.
\end{equation}
Since \(\widehat q_0^N\to q_0\) and \(\widehat p^M\to p\) in \(\mathcal W_2\) almost surely,
the sequences \(m_2(\widehat q_0^N)\) and \(m_2(\widehat p^M)\) are bounded. Hence there exists
\(C_T>0\) such that
\[
\sup_{N,M}\sup_{t\in[0,T]}\int \|x\|^2\,d\widehat q_t^{N,M}(x)\le C_T.
\]
From \eqref{eq:second-moment-bound} and the convergence in \eqref{eq:conv}, passing to the limit also gives
\[
\sup_{t\in[0,T]}\int \|x\|^2\,d\overline q_t(x)\le C_T.
\]

It remains to pass to the limit in the weak formulation. Fix \(\varphi\in C_c^1(\mathbb R^d)\).
For each \(s\in[0,T]\), let \(\gamma_s^{N,M}\) be an optimal coupling of
\(\widehat q_s^{N,M}\) and \(\overline q_s\). Write
\begin{align*}
I_{N,M}(s)
&:=\int \nabla\varphi(x)\cdot
V_{\widehat q_s^{N,M},\widehat p^M} (x)\,d\widehat q_s^{N,M}(x),\\
I(s)
&:=\int \nabla\varphi(y)\cdot
V_{\overline q_s,p} (y)\,d\overline q_s(y).
\end{align*}
Then
\begin{align*}
|I_{N,M}(s)-I(s)|
&\le
\left|\int
\bigl(\nabla\varphi(x)-\nabla\varphi(y)\bigr)\cdot
V_{\widehat q_s^{N,M},\widehat p^M} (x)
\,d\gamma_s^{N,M}(x,y)\right|\\
&\quad+
\|\nabla\varphi\|_\infty
\left(
\int
\bigl\|V_{\widehat q_s^{N,M},\widehat p^M} (x)
      -V_{\overline q_s,p} (y)\bigr\|^2
\,d\gamma_s^{N,M}(x,y)
\right)^{1/2}.
\end{align*}
The second term is bounded by
\[
\|\nabla\varphi\|_\infty
\Bigl[(L_1+L_2)\mathcal W_2(\widehat q_s^{N,M},\overline q_s)
+L_3\mathcal W_2(\widehat p^M,p)\Bigr],
\]
and converges to zero uniformly in \(s\). For the first term, let \(\omega_\varphi\) be a modulus of
continuity of \(\nabla\varphi\). For every \(r>0\), Cauchy--Schwarz and the optimality of
\(\gamma_s^{N,M}\) give
\begin{align*}
&\left|\int
\bigl(\nabla\varphi(x)-\nabla\varphi(y)\bigr)\cdot
V_{\widehat q_s^{N,M},\widehat p^M} (x)
\,d\gamma_s^{N,M}(x,y)\right|\\
&\qquad\le
\omega_\varphi(r)\,
\|V_{\widehat q_s^{N,M},\widehat p^M} \|_{L^1(\widehat q_s^{N,M})}
+2\|\nabla\varphi\|_\infty
\|V_{\widehat q_s^{N,M},\widehat p^M} \|_{L^2(\widehat q_s^{N,M})}
\frac{\mathcal W_2(\widehat q_s^{N,M},\overline q_s)}{r}.
\end{align*}
The velocity norms are uniformly bounded on \([0,T]\) by \eqref{ass:L2growth} and
\eqref{eq:second-moment-bound}. Taking \(N,M\to\infty\) and then \(r\downarrow0\) proves that
\(I_{N,M}\to I\) uniformly on \([0,T]\). Passing to the limit in the weak formulation gives
\[
\int \varphi\,d\overline q_t-\int \varphi\,dq_0
=
\int_0^t\int \nabla\varphi(x)\cdot V_{\overline q_s,p}(x)\,d\overline q_s(x)\,ds.
\]
Thus \(\overline q\) is a weak solution of \eqref{eq:mf_pde}. By uniqueness, \(\overline q\) is the
unique weak solution; we denote it by \(q\).

Applying the stability estimate in \eqref{eq:dobrushin} with
\[
\mu_t=q_t,
\qquad
\nu_t=\widehat{q}_t^{N,M},
\qquad
\widetilde p=\widehat p^M
\]
gives, for all \(t\in[0,T]\),
\begin{equation}
\label{eq:W1-t-emp}
\mathcal W_2(\widehat{q}_t^{N,M},q_t)
\le
e^{(L_1+L_2)t}\mathcal W_2(\widehat q_0^N,q_0)
+
\frac{L_3}{L_1+L_2}(e^{(L_1+L_2)t}-1)\mathcal W_2(\widehat p^M,p).
\end{equation}
Taking the supremum over \(t\in[0,T]\) gives the corresponding bound for the continuous-time particle system.

On the other hand, let \(\pmb{x}^{N,M,\eta}(t)\) be the piecewise linear Euler interpolation of \eqref{eq:odeN}, namely
\[
\pmb{x}^{N,M,\eta}(t_{k+1})
=
\pmb{x}^{N,M,\eta}(t_k)+\eta\,V^{N,M}(\pmb{x}^{N,M,\eta}(t_k)),
\]
and
\[
\pmb{x}^{N,M,\eta}(t)
=
\pmb{x}^{N,M,\eta}(t_k)+(t-t_k)V^{N,M}(\pmb{x}^{N,M,\eta}(t_k)),
\qquad t\in[t_k,t_{k+1}).
\]
Let
\[
\widehat q_t^{N,M,\eta}:=\frac1N\sum_{i=1}^N\delta_{x_i^{N,M,\eta}(t)}.
\]

Set
\[
e(t):=\|\pmb{x}^{N,M,\eta}(t)-\pmb{x}^{N,M}(t)\|_{2,N}.
\]

Since \(V^{N,M}\) is globally Lipschitz with constant \(L_1+L_2\), the standard Euler estimate gives
\[
\sup_{t\in[0,T]}e(t)
\le
\frac12\bigl(e^{(L_1+L_2)T}-1\bigr)\eta
\sup_{s\in[0,T]}
\|V^{N,M}(\pmb{x}^{N,M}(s))\|_{2,N}.
\]
Moreover, by \eqref{ass:L2growth} and \eqref{eq:second-moment-bound}, for every \(s\in[0,T]\),
\[
\begin{aligned}
\|V^{N,M}(\pmb{x}^{N,M}(s))\|_{2,N}
&=
\|V_{\widehat q_s^{N,M},\widehat p^M}^{\varepsilon}\|_{L^2(\widehat q_s^{N,M})}  \\
&\le
L_0\bigl(m_2(\widehat q_s^{N,M})+m_2(\widehat p^M)\bigr) \\
&\le
L_0 e^{L_0T}
\bigl[m_2(\widehat q_0^N)+m_2(\widehat p^M)\bigr].
\end{aligned}
\]
Therefore,
\begin{equation}
\label{eq:euler-error-bound}
\sup_{t\in[0,T]} e(t)
\le
\frac{L_0 e^{L_0T}}{2}
\bigl(e^{(L_1+L_2)T}-1\bigr)
\eta
\bigl[m_2(\widehat q_0^N)+m_2(\widehat p^M)\bigr].
\end{equation}

Since the coupling that matches \(x_i^{N,M,\eta}(t)\) with
\(x_i^{N,M}(t)\) is admissible, we have
\[
\mathcal W_2(\widehat q_t^{N,M,\eta},\widehat q_t^{N,M})
\le
\|\pmb{x}^{N,M,\eta}(t)-\pmb{x}^{N,M}(t)\|_{2,N}
=
e(t).
\]
Combining this estimate with \eqref{eq:W1-t-emp} and
\eqref{eq:euler-error-bound}, and taking the supremum over \(t\in[0,T]\), proves the quantitative bound in the theorem.

Finally, since \(q_0,p\in\mathcal P_2(\mathbb R^d)\), the empirical measures
satisfy
\[
\mathcal W_2(\widehat q_0^N,q_0)\to0,
\qquad
\mathcal W_2(\widehat p^M,p)\to0,
\qquad\text{almost surely}.
\]
Moreover,
\begin{align*}
\int_{\mathbb R^d}\|x\|^2\,d\widehat q_0^N(x)
\to
\int_{\mathbb R^d}\|x\|^2\,dq_0(x), \quad
\int_{\mathbb R^d}\|x\|^2\,d\widehat p^M(x)
\to
\int_{\mathbb R^d}\|x\|^2\,dp(x),
\qquad\text{almost surely}.
\end{align*}
Letting \(N,M\to\infty\) and \(\eta\downarrow0\), we obtain
\[
\sup_{t\in[0,T]}
\mathcal W_2(\widehat q_t^{N,M,\eta},q_t)
\to0,
\qquad\text{almost surely}.
\]
\end{proof}

\subsection{Assumption verification for Sinkhorn}
Recall that the Sinkhorn velocity in \eqref{eq:continuous-sinkhorn-drift}:
\[
V_{q,p}^{\varepsilon}(x)
=
T_{q,p}^{\varepsilon}(x)-T_{q,q}^{\varepsilon}(x),
\]
where \(T_{q,p}^{\varepsilon}\) denotes the barycentric projection of the entropic optimal coupling
between \(q\) and \(p\) for the quadratic cost.

Before verifying Assumption~\ref{ass}, we first localize the dynamics. If
\(\operatorname{supp}(q_0)\subset B_{R_0}\) and
\(\operatorname{supp}(p)\subset B_R\), then the WGF under Sinkhorn divergence remains supported in a bounded
ball on every finite time interval; see Proposition~\ref{prop:sinkhorn-support-propagation}. On this bounded domain, we then use compact-domain
stability and regularity estimates for Sinkhorn barycentric projections, collected in
Lemma~\ref{lem:sinkhorn-barycentric-compact}.  Combining these two ingredients allows us to verify
Assumption~\ref{ass} for the Sinkhorn velocity along the localized flow and its finite-particle
approximations; see Proposition~\ref{prop:sinkhorn-bounded-domain}.

\begin{Proposition}[Bounded support propagation for the Sinkhorn flow]
\label{prop:sinkhorn-support-propagation}
Assume that
\[
\operatorname{supp}(q_0)\subset B_{R_0},
\qquad
\operatorname{supp}(p)\subset B_R,
\]
where \(B_R:=\{x\in\mathbb R^d:\|x\|\le R\}\). Let \(q_t\) solve
\[
\partial_t q_t+\nabla\cdot(q_tV_{q_t,p}^{\varepsilon})=0.
\]
Then, on every finite time interval \([0,T]\),
\[
\operatorname{supp}(q_t)\subset B_{e^tR_0+(e^t-1)R}.
\]
\end{Proposition}
The same support bound holds for the discrete finite-particle dynamics in \eqref{eq:euler_particles} and for the
piecewise-linear interpolation in \eqref{eq:linear_interpolation}, provided the initial particles are supported in \(B_{R_0}\)
and the target particles are supported in \(B_R\).
\begin{proof}
Let us define the radius
\[
\rho(t):=\sup\{\|x\|:x\in\operatorname{supp}(q_t)\}.
\]
Since \(T_{q_t,p}^{\varepsilon}(x)\) is the barycenter of a probability measure supported in
\(\operatorname{supp}(p)\subset B_R\), and \(T_{q_t,q_t}^{\varepsilon}(x)\) is the barycenter of a
probability measure supported in \(\operatorname{supp}(q_t)\), we have
\[
\|T_{q_t,p}^{\varepsilon}(x)\|\le R,
\qquad
\|T_{q_t,q_t}^{\varepsilon}(x)\|\le \rho(t).
\]
Therefore,
\[
\|V_{q_t,p}^{\varepsilon}(x)\|
\le
\|T_{q_t,p}^{\varepsilon}(x)\|
+
\|T_{q_t,q_t}^{\varepsilon}(x)\|
\le
R+\rho(t),
\qquad x\in\operatorname{supp}(q_t).
\]
Equivalently, for any particle trajectory \(x(t)\) in the support of the flow,
\[
\frac{d}{dt}\|x(t)\|
\le
\|V_{q_t,p}^{\varepsilon}(x(t))\|
\le
R+\rho(t).
\]
Taking the supremum over all such trajectories gives
\[
\rho(t)
\le
R_0+\int_0^t(R+\rho(s))\,ds.
\]
By Gronwall's inequality,
\[
\rho(t)
\le
e^tR_0+(e^t-1)R.
\]
This proves
\[
\operatorname{supp}(q_t)
\subset
B_{e^tR_0+(e^t-1)R}.
\]

For the continuous-time finite-particle dynamics, the same argument applies with \(q_t\) replaced by
the empirical measure. For the Euler scheme, if \(\widehat q_k^N\) is supported in \(B_{\rho_k}\),
then
\[
\left\|x_i^{(k+1)}\right\|
\le
\left\|x_i^{(k)}\right\|
+
\eta \left\|V_{\widehat q_k^N,\widehat p^M}^{\varepsilon}(x_i^{(k)})\right\|
\le
\rho_k+\eta(R+\rho_k).
\]
Hence
\[
\rho_{k+1}\le (1+\eta)\rho_k+\eta R.
\]
Iterating this recursion gives
\[
\rho_k
\le
(1+\eta)^kR_0+\bigl((1+\eta)^k-1\bigr)R
\le
e^{t_k}R_0+(e^{t_k}-1)R.
\]
The same estimate controls the piecewise-linear interpolation between \(t_k\) and \(t_{k+1}\).
\end{proof}

\begin{Lemma}[Compact-domain stability of Sinkhorn barycentric projections]
\label{lem:sinkhorn-barycentric-compact}
Let \(K\subset\mathbb R^d\) be compact, let \(D_K:=\operatorname{diam}(K)\), and fix
\(\varepsilon>0\). There exists \(C_{K,\varepsilon}<\infty\) such that, for all
\(q,\widetilde q,p,\widetilde p\in\mathcal P(K)\),
\begin{equation}
\label{eq:compact-T-measure-stability}
\sup_{x\in K}
\bigl\|T_{q,p}^{\varepsilon}(x)-T_{\widetilde q,\widetilde p}^{\varepsilon}(x)\bigr\|
\le
C_{K,\varepsilon}
\bigl(\mathcal W_2(q,\widetilde q)+\mathcal W_2(p,\widetilde p)\bigr),
\end{equation}
and, for all \(x,x'\in K\),
\begin{equation}
\label{eq:compact-T-spatial-stability}
\bigl\|T_{q,p}^{\varepsilon}(x)-T_{q,p}^{\varepsilon}(x')\bigr\|
\le
\frac{D_K^2}{4\varepsilon}\|x-x'\|.
\end{equation}
\end{Lemma}

\begin{proof}
We first prove \eqref{eq:compact-T-measure-stability}. By the compact-domain stability result for
entropic OT potentials \cite[Corollary~2.4]{carlier2024displacement}, applied to the rescaled
quadratic cost \((x,y)\mapsto \|x-y\|^2/(2\varepsilon)\) on \(K\times K\), there exists
\(C_{K,\varepsilon}<\infty\) such that, for all
\(q,\widetilde q,p,\widetilde p\in\mathcal P(K)\),
\[
\sup_{x\in K}
\bigl\|
\nabla u_{q,p}^{\varepsilon}(x)
-
\nabla u_{\widetilde q,\widetilde p}^{\varepsilon}(x)
\bigr\|
\le
C_{K,\varepsilon}
\bigl(
\mathcal W_2(q,\widetilde q)
+
\mathcal W_2(p,\widetilde p)
\bigr),
\]
where \(u^\varepsilon_{q,p}\) denotes the Schr\"odinger potential introduced in
\eqref{eq:sinkhorn_dual}. Since \(K\) is compact and the quadratic cost is smooth on \(K\times K\),
the assumptions required for \cite[Corollary~2.4]{carlier2024displacement} to hold are satisfied. Using the identity in
\eqref{eq:dual_grad_representation}, we immediately obtain \eqref{eq:compact-T-measure-stability}.

It remains to relate the Schr\"odinger potential to the barycentric projection. Recall from \eqref{eq:dual_grad_representation} that $\nabla u_{q,p}^{\varepsilon}(x)=x-T_{q,p}^{\varepsilon}(x)$.
Therefore,
\[
\begin{aligned}
\sup_{x\in K}
\bigl\|
T_{q,p}^{\varepsilon}(x)
-
T_{\widetilde q,\widetilde p}^{\varepsilon}(x)
\bigr\|
&=
\sup_{x\in K}
\bigl\|
\nabla u_{q,p}^{\varepsilon}(x)
-
\nabla u_{\widetilde q,\widetilde p}^{\varepsilon}(x)
\bigr\|  \\
&\le
C_{K,\varepsilon}
\bigl(
\mathcal W_2(q,\widetilde q)
+
\mathcal W_2(p,\widetilde p)
\bigr).
\end{aligned}
\]
This proves \eqref{eq:compact-T-measure-stability}.

It remains to prove the spatial estimate. Since
\[
T_{q,p}^{\varepsilon}(x)
=
\int_K y\,\pi_{q,p}^{\varepsilon}(dy\mid x),
\]
differentiating the conditional mean with respect to \(x\) gives
\[
\nabla_xT_{q,p}^{\varepsilon}(x)
=
\frac1{\varepsilon}
\operatorname{Cov}_{\pi_{q,p}^{\varepsilon}(dy\mid x)}(Y).
\]
For any probability measure supported on \(K\), the operator norm of its covariance matrix is at
most \(D_K^2/4\). Therefore
\[
\|\nabla_xT_{q,p}^{\varepsilon}(x)\|_{\mathrm{op}}
\le
\frac{D_K^2}{4\varepsilon},
\qquad x\in K.
\]
This proves \eqref{eq:compact-T-spatial-stability}.
\end{proof}

\begin{Proposition}[Verification of Assumption~\ref{ass} for bounded-support Sinkhorn flow]
\label{prop:sinkhorn-bounded-domain}
Assume
\[
\operatorname{supp}(q_0)\subset B_{R_0},
\qquad
\operatorname{supp}(p)\subset B_R,
\]
and fix \(T>0\). Then all measures appearing in the Sinkhorn flow and in the finite-particle
approximations on \([0,T]\) are supported in $B_{\bar R_T}$ with
\[
\bar R_T:=\max\{R,\ e^TR_0+(e^T-1)R\}.
\]
Moreover, for the Sinkhorn velocity $V_{q,p}^{\varepsilon}(x)$ defined in \eqref{eq:continuous-sinkhorn-drift}, Assumption~\ref{ass} holds on \(\mathcal P(B_{\bar R_T})\), with constants depending only on
\((T,R_0,R,\varepsilon)\). More explicitly, there exist
\(L_0,L_1,L_2,L_3<\infty\), depending only on \((\bar R_T,\varepsilon)\), such that
\eqref{ass:L2growth}--\eqref{ass:L2p} hold for all
\(q,\widetilde q,p,\widetilde p\in\mathcal P(B_{\bar R_T})\), with the spatial estimate required
only for \(x,x'\in B_{\bar R_T}\).
\end{Proposition}

\begin{proof}
The support statement follows from Proposition~\ref{prop:sinkhorn-support-propagation}. It remains
to verify Assumption~\ref{ass} on \(B_{\bar R_T}\).

The \(L^2\)-growth estimate follows from Jensen's inequality. Specifically, if
\((X,Y)\sim\pi_{q,p}^{\varepsilon}\), then
\[
\int \|T_{q,p}^{\varepsilon}(x)\|^2\,dq(x)
=
\int \|\mathbb E[Y\mid X=x]\|^2\,dq(x)
\le
\int \|y\|^2\,dp(y).
\]
Similarly,
\[
\int \|T_{q,q}^{\varepsilon}(x)\|^2\,dq(x)
\le
\int \|y\|^2\,dq(y).
\]
Hence
\[
\|V_{q,p}^{\varepsilon}\|_{L^2(q)}
\le
\|T_{q,p}^{\varepsilon}\|_{L^2(q)}
+
\|T_{q,q}^{\varepsilon}\|_{L^2(q)}
\le
m_2(p)+m_2(q),
\]
so \eqref{ass:L2growth} holds with \(L_0=1\).

By Lemma~\ref{lem:sinkhorn-barycentric-compact}, for \(x,x'\in B_{\bar R_T}\),
\[
\begin{aligned}
\|V_{q,p}^{\varepsilon}(x)-V_{q,p}^{\varepsilon}(x')\|
&\le
\|T_{q,p}^{\varepsilon}(x)-T_{q,p}^{\varepsilon}(x')\|
+
\|T_{q,q}^{\varepsilon}(x)-T_{q,q}^{\varepsilon}(x')\|  \\
&\le
\frac{\operatorname{diam}(B_{\bar R_T})^2}{2\varepsilon}\|x-x'\|.
\end{aligned}
\]
Thus \eqref{ass:spatialLip} holds on \(B_{\bar R_T}\).

For the \(q\)-stability estimate, Lemma~\ref{lem:sinkhorn-barycentric-compact} gives
\[
\begin{aligned}
\sup_{x\in B_{\bar R_T}}
\|V_{q,p}^{\varepsilon}(x)-V_{\widetilde q,p}^{\varepsilon}(x)\|
&\le
\sup_{x\in B_{\bar R_T}}
\|T_{q,p}^{\varepsilon}(x)-T_{\widetilde q,p}^{\varepsilon}(x)\|  \\
&\quad+
\sup_{x\in B_{\bar R_T}}
\|T_{q,q}^{\varepsilon}(x)-T_{\widetilde q,\widetilde q}^{\varepsilon}(x)\| \\
&\le
3C_{B_{\bar R_T},\varepsilon}\mathcal W_2(q,\widetilde q).
\end{aligned}
\]
Integrating with respect to \(\widetilde q\) gives \eqref{ass:L2q}.

For the \(p\)-stability estimate, the self-interaction term cancels:
\[
V_{q,p}^{\varepsilon}(x)-V_{q,\widetilde p}^{\varepsilon}(x)
=
T_{q,p}^{\varepsilon}(x)-T_{q,\widetilde p}^{\varepsilon}(x).
\]
Therefore, by Lemma~\ref{lem:sinkhorn-barycentric-compact},
\[
\sup_{x\in B_{\bar R_T}}
\|V_{q,p}^{\varepsilon}(x)-V_{q,\widetilde p}^{\varepsilon}(x)\|
\le
C_{B_{\bar R_T},\varepsilon}\mathcal W_2(p,\widetilde p).
\]
Integrating with respect to \(q\) gives \eqref{ass:L2p}.
\end{proof}


\section{Additional implementation details}\label{app:implementation}

\textbf{ImageNet experiments.} We provide detailed hyperparameters and configurations in Table~\ref{tab:hypers_all}. Note that we largely follow the configurations and setups adopted in~\cite{deng2026generative}, without extensive hyperparameter tuning. We conduct experiments on 8 Nvidia H100 GPU nodes, where each GPU has 80 GB of memory. The training on L/2 takes around 3 days.

\begin{table*}[t]
    \centering
    \caption{Configurations and hyperparameters for ImageNet experiments.
    }
    \label{tab:hypers_all}
      \resizebox{1\textwidth}{!}{
    \tablestyle{1pt}{1.0}
    \scriptsize
    \begin{tabular}{l |x{80}|x{80}x{100}x{100}}
    \toprule
    & \textbf{Ablation setup} (Table~\ref{tab:fid_ablation}) & \textbf{\method B/2} (Table~\ref{tab:in256_latent}) & \textbf{\method L/2} (Table~\ref{tab:in256_latent}) & \textbf{\method XL/2} (Table~\ref{tab:in256_latent}) \\
    \midrule
    \rowcolor[gray]{0.9} \multicolumn{5}{l}
    {\textit{\textbf{Generator Architecture}}} \\
    arch & DiT-B/2 & DiT-B/2 & DiT-L/2 & DiT-XL/2 \\
    input size & 32$\times$32$\times$4 & 32$\times$32$\times$4 & 32$\times$32$\times$4 & 32$\times$32$\times$4 \\
    patch size & 2$\times$2 & 2$\times$2 & 2$\times$2 & 2$\times$2 \\
    hidden dim & 768 & 768 & 1024 & 1152 \\
    depth & 12 & 12 & 24 & 28 \\
    register tokens & 16 & 16 & 16 & 16 \\
    style embedding tokens & 32 & 32 & 32 & 32 \\
    \midrule
    \rowcolor[gray]{0.9} \multicolumn{5}{l}{\textit{\textbf{Feature Encoder}}} \\
    architecture & ResNet & ResNet & ResNet & ResNet \\
    SSL pre-train method & latent-MAE & latent-MAE & latent-MAE & latent-MAE \\
    ResNet: input size & 32$\times$32$\times$4 & 32$\times$32$\times$4 & 32$\times$32$\times$4 & 32$\times$32$\times$4 \\
    ResNet: conv$_\text{1}$ stride & 1 & 1 & 1 & 1 \\
    ResNet: base width & 256 & 640 & 640 & 640 \\
    ResNet: block type & \multicolumn{4}{c}{bottleneck} \\
    ResNet: blocks / stage & \multicolumn{4}{c}{[3, 4, 6, 3]} \\
    ResNet: size / stage & \multicolumn{4}{c}{[32$^2$, 16$^2$, 8$^2$, 4$^2$]} \\
    MAE: masking ratio & \multicolumn{4}{c}{50\%} \\
    MAE: pre-train epochs & 192 & 1280 & 1280 & 1280 \\
    classification finetune & No & 3k steps & 3k steps & 3k steps \\
    \midrule
    \rowcolor[gray]{0.9} \multicolumn{5}{l}{\textit{\textbf{Generator Optimizer}}} \\
    optimizer & \multicolumn{4}{c}{AdamW ($\beta_1$ = 0.9, $\beta_2$ = 0.95)} \\
    learning rate & 2e-4 & 4e-4 & 4e-4 & 3e-4 \\
    weight decay & 0.01 & 0.0 & 0.01 & 0.01 \\
    warmup steps & 5k & 10k & 10k & 10k \\
    gradient clip & 2.0 & 2.0 & 2.0 & 2.0 \\
    training steps & 30k & 200k & 200k & 200k \\
    training epochs & 100 & 1280 & 1280 & 1280 \\
    EMA decay & 0.999 & 0.999 & 0.999 & 0.999 \\
    \midrule
    \rowcolor[gray]{0.9} \multicolumn{5}{l}{\textit{\textbf{Training Loss Computation}}} \\
    class labels $N_\text{c}$ & 64 & 128 & 128 & 128 \\
    positive samples $N_\text{pos}$ & 64 & 64 & 64 & 128 \\
    generated samples $N_\text{neg}$ & 64 & 64 & 64 & 64 \\
    effective batch $B$ ($N_\text{c}{\times}N_\text{neg}$) & 4096 & 8192 & 8192 & 8192 \\
    iteration $L$ & 10 & 10 & 1 & 1 \\
    regularization $\varepsilon$ & 0.05 & 0.05 & 0.05 & 0.01 \\
    step size $\eta$  & 1.0 & 1.0 & 1.0 & 1.0 \\
    \midrule
    \rowcolor[gray]{0.9} \multicolumn{5}{l}{\textit{\textbf{CFG Configuration}}} \\
    train: CFG $w$ range & $[0, 3]$ & $[0, 3]$ & $[0, 3]$ & $[0, 3]$ \\
    train: CFG $w$ sampling & $p(w){\propto}(w+1)^{-3}$ & $p(w) \propto (w+1)^{-5}$ & \tiny 50\%: $w{=}0$, 50\%: $p(w){\propto} (w+1)^{-3}$ & \tiny 20\%: $w{=}0$, 80\%: $p(w){\propto} (w+1)^{-4}$ \\
    train: uncond samples $N_\text{uncond}$ & 16 & 32 & 32 & 32 \\
    inference: CFG $w$ search & \multicolumn{4}{c}{[0.0, 2.5]} \\
    \bottomrule
    \end{tabular}
    }
\end{table*}

\textbf{FFHQ experiments.} For the FFHQ experiments, we operate in the 512-dimensional latent space of a pretrained Adversarial Latent Autoencoder (ALAE), utilizing a training split of 60,000 images. The generator is parameterized as a 4-layer MLP with hidden dimensions of 1024, utilizing LayerNorm and SiLU activations in the hidden blocks. Models are trained using the AdamW optimizer with a constant learning rate of $2 \times 10^{-4}$, $\beta_1=0.9$, $\beta_2=0.95$, and zero weight decay. We apply gradient clipping at 5.0 and maintain an exponential moving average (EMA) of the weights with a decay rate of 0.999. During training, we use a batch size of $N=64$ for the generated particles and $M=64$ for the target particles, with a fixed step size of $\eta=1$. For the Sinkhorn divergence, we use the quadratic cost with an entropy regularization parameter $\varepsilon=0.02$, with 10 Sinkhorn iterations. Consistent with Sec.~\ref{subsec:particle}, we explicitly use the two-batch debiased self-transport estimator, avoiding any heuristic diagonal masking.

We evaluate this setup on two distinct tasks: domain transfer and mode coverage, as shown in Sec.~\ref{subsec:other_tasks}. For the domain transfer task, we map a source distribution of senior faces (ages 55--100) to a target distribution of young adult faces (ages 18--30). The initial particles are drawn directly from the source data distribution, and we apply the residual identity initialization by zero-initializing the final layer of the MLP to learn the residual update. For the mode coverage experiment, we perform unconditional generation starting from a standard Gaussian prior. The target batches are constructed to be highly imbalanced, comprising 95\% senior faces and 5\% child faces (ages 0--12). In this setting, the standard network initialization is used without the residual connection.

\textbf{Settings for plotting Fig.~\ref{fig:title-fig} (b).}
We train~\method with the Sinkhorn divergence on toy datasets for 20,000 iterations, saving a checkpoint every 400 iterations. This yields a sequence of parameter vectors \(\{\theta_k\}_{k=0}^{K}\), where \(K=50\).
We visualize the Sinkhorn-divergence energy landscape over a two-dimensional affine slice of the parameter space. Each point \((x,y)\) on the $x$-$y$ plane is mapped to a parameter vector \(\theta_{x,y}\), and the corresponding height is given by
$S_\varepsilon(q_{\theta_{x,y}}, p),$
where \(q_{\theta_{x,y}}\) denotes the model distribution induced by \(\theta_{x,y}\).

We construct the $x$-$y$ grid as follows. First, we place the training trajectory along the \(x\)-axis by assigning checkpoint \(\theta_k\) to coordinate \((k,0)\) for \(k=0,\ldots,K\). Let
\[
    d_k = \theta_{k+1} - \theta_k, \qquad k=0,\ldots,K-1,
\]
denote the displacement directions along the checkpoint trajectory. We then sample a random direction \(v\) in parameter space and project it onto the orthogonal complement of \(\mathrm{span}\{d_0,\ldots,d_{K-1}\}\), followed by normalization. This direction \(v\) is used as a common transverse direction for the \(y\)-axis.

For each interval \(x\in[k,k+1]\), we interpolate along the local trajectory direction \(d_k\) and extend transversely along \(v\):
\[
    \theta_{x,y}
    =
    \theta_k
    +
    (x-k)d_k
    +
    yv,
    \qquad
    k \le x \le k+1,\quad -Y \le y \le Y .
\]
This defines the landscape grid over \(0\le x\le K\) and \(-Y\le y\le Y\), while ensuring that the saved checkpoints lie exactly on the \(x\)-axis.

To visualize the landscape beyond the training trajectory, we further extend the \(x\)-axis to \(-X\le x\le K+X\). For the left extension \(-X\le x\le 0\), we use the initial trajectory direction \(d_0=\theta_1-\theta_0\):
\[
    \theta_{x,y}
    =
    \theta_0 + x d_0 + yv .
\]
For the right extension \(K\le x\le K+X\), we use the final trajectory direction \(d_{K-1}=\theta_K-\theta_{K-1}\):
\[
    \theta_{x,y}
    =
    \theta_K + (x-K)d_{K-1} + yv .
\]
Together, this construction yields a continuous two-dimensional slice of the parameter space on which we evaluate and plot the Sinkhorn-divergence energy.

\section{Sampling throughput comparison}
\label{app:throughput}
We compare our one-step~\method against multi-step diffusion model SiT~\cite{ma2024sit} and LightningDiT~\cite{yao2025reconstruction} in terms of actual sampling throughput measured by wall-clock time. We benchmark on a single Nvidia H100 GPU with a batch size of 4 for all models. For SiT-XL/2 and LightningDiT-XL/2, we adopt 250 sampling steps with CFG enabled and strictly follow the configuration in the original papers~\cite{ma2024sit,yao2025reconstruction} that leads to a reported FID of 2.06 and 1.35, respectively. Note that we also include the latent preparation and VAE decoding time when recording the actual runtime and calculating the throughput. The results are depicted in Table~\ref{tab:throughput_comparison}, where~\method clearly yields significant speedup while achieving competitive generation quality measured by FID on ImageNet 256$\times$256.

\begin{table}[htbp]
    \centering
    \small
    \caption{\textbf{Generation quality and throughput comparison.}
    Throughput is measured in \emph{images per second}. The time is recorded as the \emph{actual wallclock time}. The speedup is computed \emph{w.r.t.} SiT-XL/2.}
    \label{tab:throughput_comparison}
    \begin{tabular}{lccc}
        \toprule
         & FID $\downarrow$ & Throughput (images/sec) $\uparrow$ & Speedup $\uparrow$ \\
        \midrule
        \method, B/2 & 1.52 & 107.76 & 115.87$\times$ \\
        \method, L/2 & 1.35 & 84.88 & 91.27$\times$ \\
        \method, XL/2 & 1.29 & 77.88 & 83.74$\times$ \\
        LightningDiT-XL/2~\cite{yao2025reconstruction} & 1.35 & 1.14 & 1.22$\times$ \\
        SiT-XL/2~\cite{ma2024sit} & 2.06 & 0.93 & -- \\
        \bottomrule
    \end{tabular}
\end{table}

\section{Additional experimental results}
\label{app:additional-ablation}
\textbf{Sinkhorn iterations.} We provide more ablation results on number of Sinkhorn iterations $L$ used when computing the Sinkhorn barycentric projection (Alg.~\ref{alg:sinkhorn_projection}). We follow the same ablation setup in Sec.~\ref{exp:abl} with distribution guidance. The results in Table~\ref{tab:ablation_sinkhorn_iter} show that performance improves as the number of Sinkhorn iterations increases, plateauing at around 10 iterations.

\begin{table}[htbp]
    \centering
    \begin{minipage}[t]{0.3\textwidth}
    \small
        \centering
        \caption{Ablation on Sinkhorn iterations $L$.}
        \label{tab:ablation_sinkhorn_iter}
        \begin{tabular}{c c}
            \toprule
            $L$ & FID, 1-NFE $\downarrow$ \\
            \midrule
            1  & 7.92 \\
            5 & 7.45 \\
            10 & 7.29 \\
            20 & 7.33 \\
            \bottomrule
        \end{tabular}
    \end{minipage}\hfill
    \begin{minipage}[t]{0.3\textwidth}
    \small
        \centering
        \caption{Velocity guidance vs. distribution guidance applied to KL divergence.}
        \label{tab:ablation_kl_velo_guidance}
        \begin{tabular}{c c}
            \toprule
            CFG & FID, 1-NFE $\downarrow$ \\
            \midrule
            \texttt{dist.}  & 10.17 \\
            \texttt{velo.} & 7.68 \\
            \bottomrule
        \end{tabular}
    \end{minipage}\hfill
    \begin{minipage}[t]{0.3\textwidth}
         \small
        \centering
        \caption{Ablation on $\varepsilon$ with $\ell_2$ OT cost.}
        \label{tab:more-epsilon-distance}
        \begin{tabular}{c c}
            \toprule
            $\varepsilon$ & FID, 1-NFE $\downarrow$ \\
            \midrule
            0.025 & 7.19 \\
            0.05 & 7.16 \\
            0.075 & 7.40 \\
            \bottomrule
        \end{tabular}
    \end{minipage}
    
\end{table}

\textbf{Velocity guidance.} We further show in Table~\ref{tab:ablation_kl_velo_guidance} that the proposed velocity guidance not only enhances the performance of Sinkhorn divergence-driven WGF, but also that of KL divergence. This observation aligns with our theoretical analysis in Appendix~\ref{app:kl_velocity_guidance}, where we draw a connection to the exponentially tilted target distribution when using velocity guidance with KL divergence. Note that KL divergence with velocity guidance also outperforms the design in~\cite{deng2026generative} with 8.46 FID, further highlighting the significance of our principled WGF framework.

\textbf{Entropic regularization parameter $\varepsilon$.} We provide more results under the ablation study setting for the entropic regularization parameter $\varepsilon$ with the $\ell_2$ distance cost, under velocity guidance. As shown in Table~\ref{tab:more-epsilon-distance}, the $\ell_2$ distance OT cost yields consistently higher FIDs than the quadratic OT cost, while both achieve the best performance under $\varepsilon=0.05$.

\section{Additional generated samples}\label{app:generated_samples}
In this section, we provide additional generated image samples. We provide uncurated samples on ImageNet 256$\times$256 in Fig.~\ref{fig:imagenet-more-samples-L} and Fig.~\ref{fig:imagenet-more-samples-XL-and-B}, and additional samples in Fig.~\ref{fig:imagenet-more-samples-XL-largecfg} with a larger guidance scale. We provide uncurated samples on FFHQ for the mode coverage experiment in Fig.~\ref{fig:ffhq-more-drifting} and Fig.~\ref{fig:ffhq-more-ours} for Drifting Model and our~\method, respectively.

\begin{figure}[t!]
    \centering
    \includegraphics[width=0.99\linewidth]{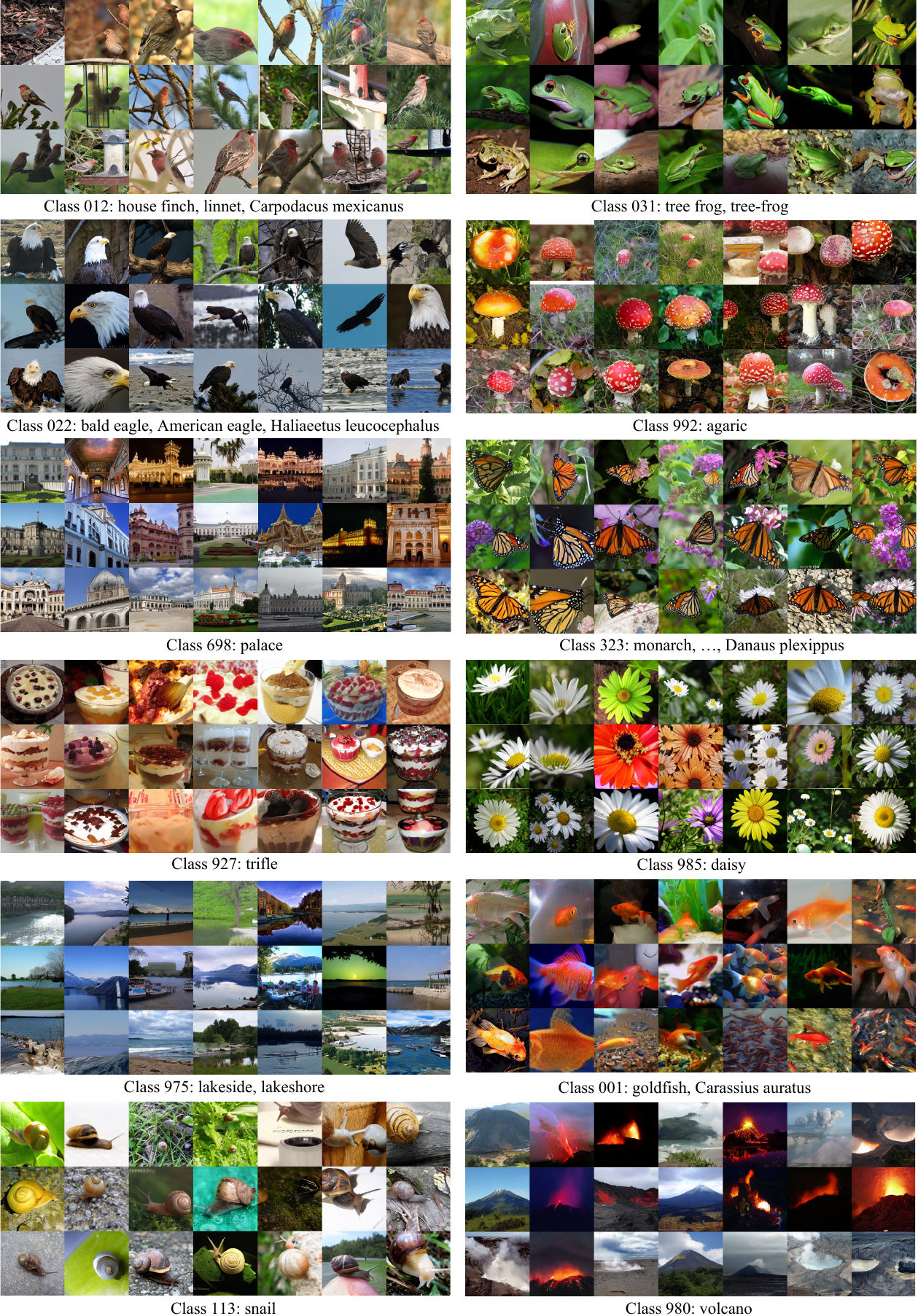}
    \caption{Uncurated samples generated by~\method, L/2 with CFG $w=0.15$.}
    \label{fig:imagenet-more-samples-L}
\end{figure}

\begin{figure}[t!]
    \centering
    \includegraphics[width=0.99\linewidth]{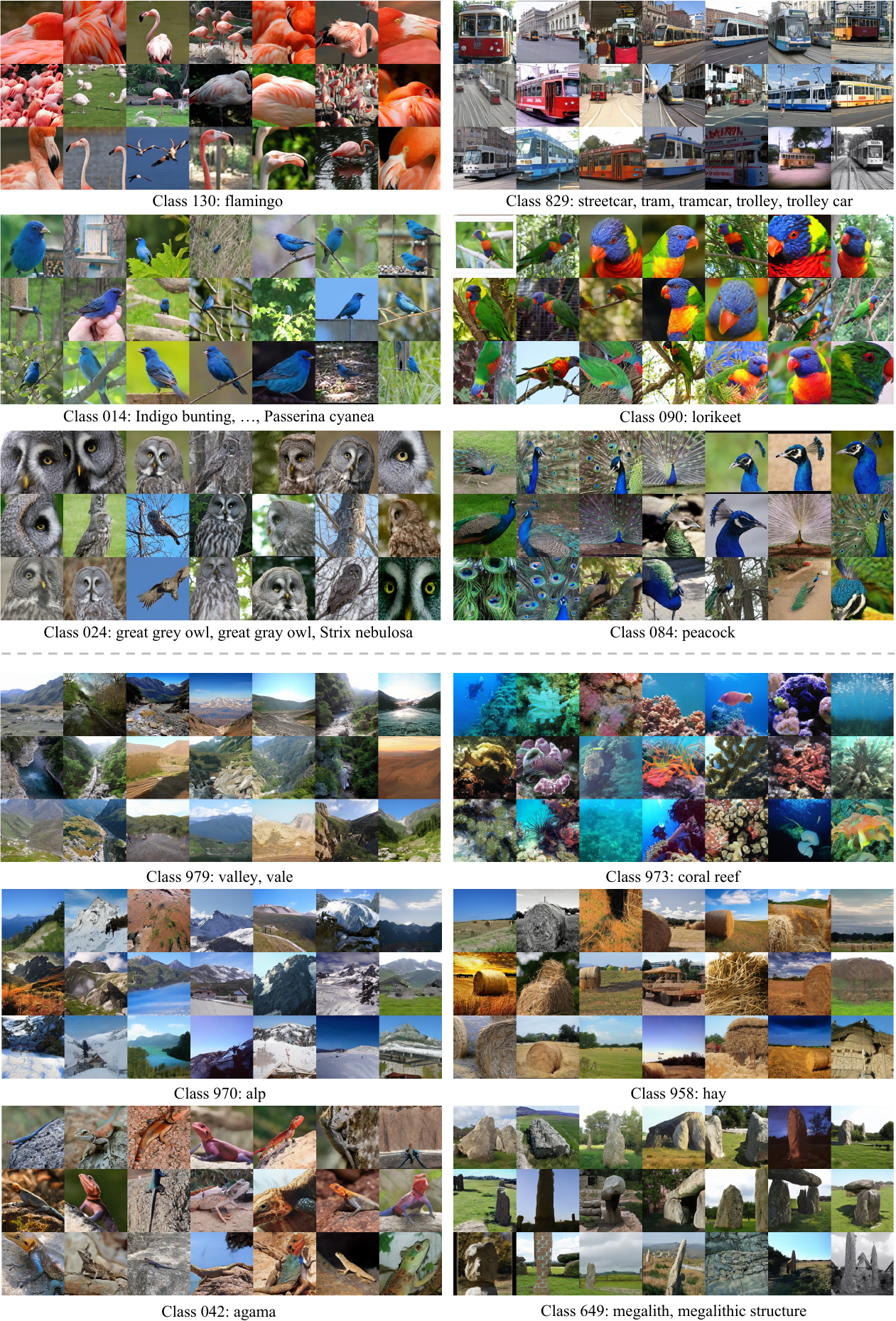}
    \caption{Uncurated samples generated by~\method, XL/2 \emph{(Top)} and B/2 \emph{(Bottom)} with CFG $w=0.15$.}
    \label{fig:imagenet-more-samples-XL-and-B}
\end{figure}

\begin{figure}[t!]
    \centering
    \includegraphics[width=0.94\linewidth]{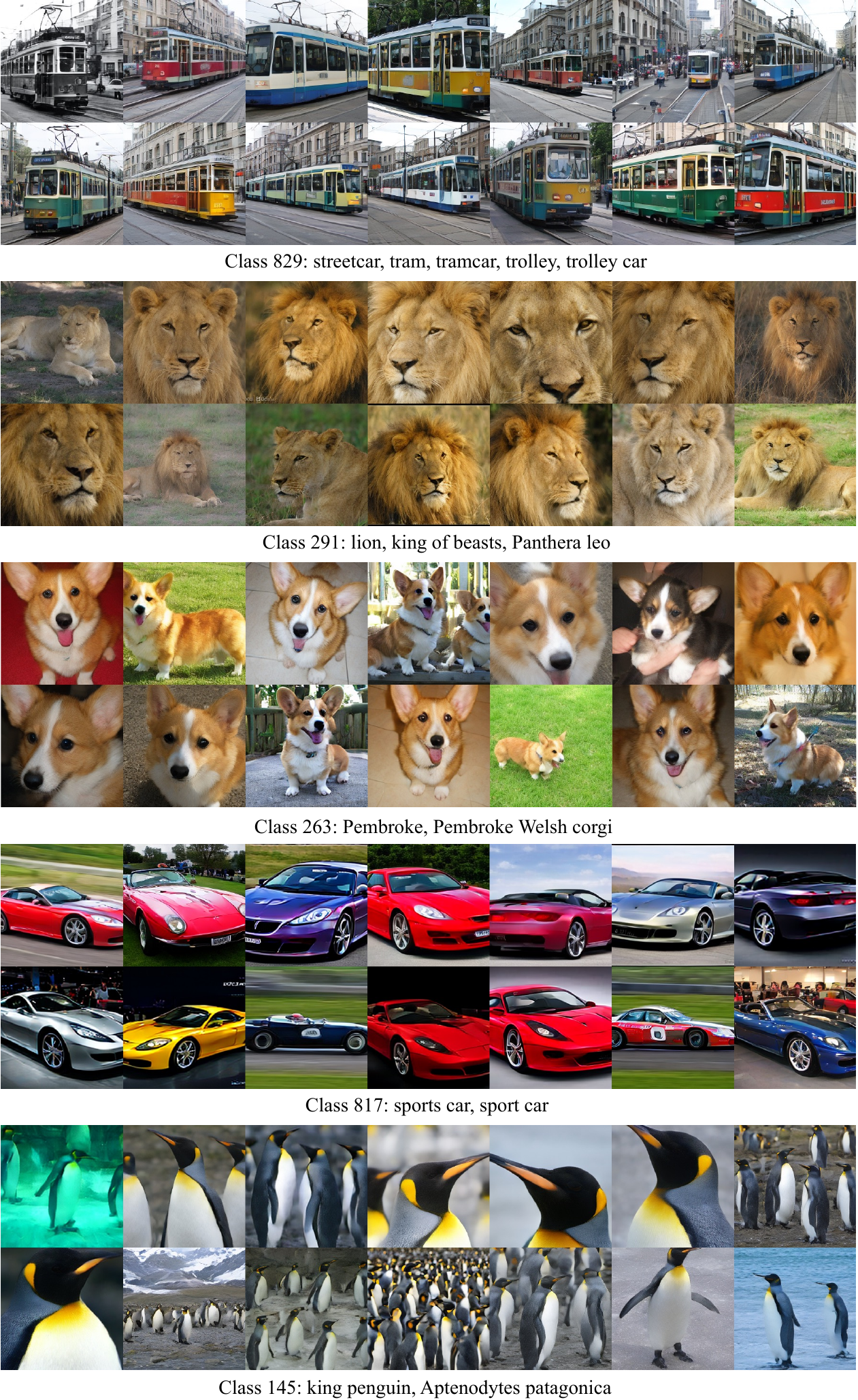}
    \caption{Uncurated samples generated by~\method, XL/2 with CFG $w=2.0$.}
    \label{fig:imagenet-more-samples-XL-largecfg}
\end{figure}

\begin{figure}[t!]
    \centering
    \includegraphics[width=1.0\linewidth]{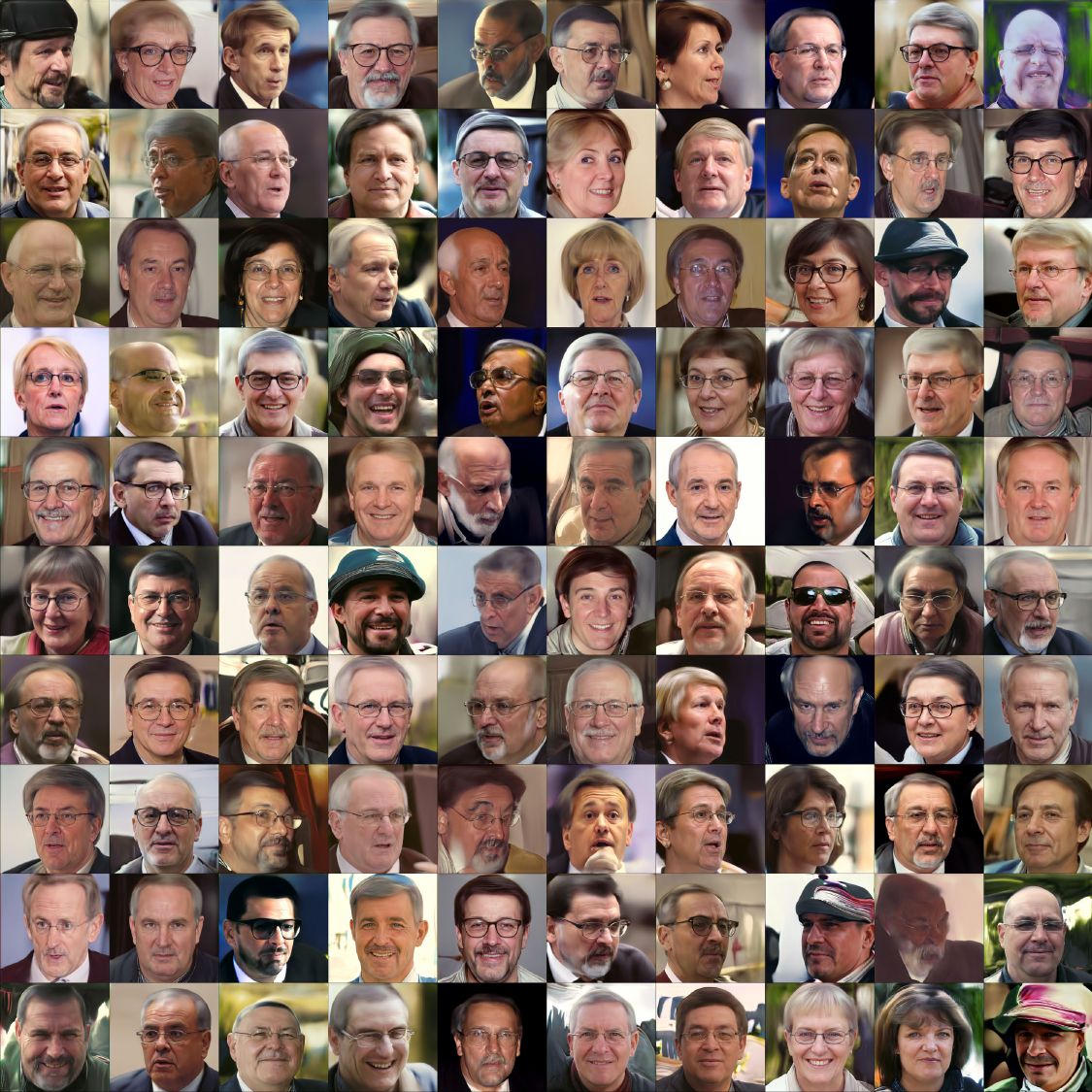}
    \caption{Uncurated samples generated by Drifting Model in the mode coverage experiment (Sec.~\ref{subsec:other_tasks}). None of the 100 randomly generated images is a child's face, showing failure in covering the minority mode.}
    \label{fig:ffhq-more-drifting}
\end{figure}

\begin{figure}[t!]
    \centering
    \includegraphics[width=1.0\linewidth]{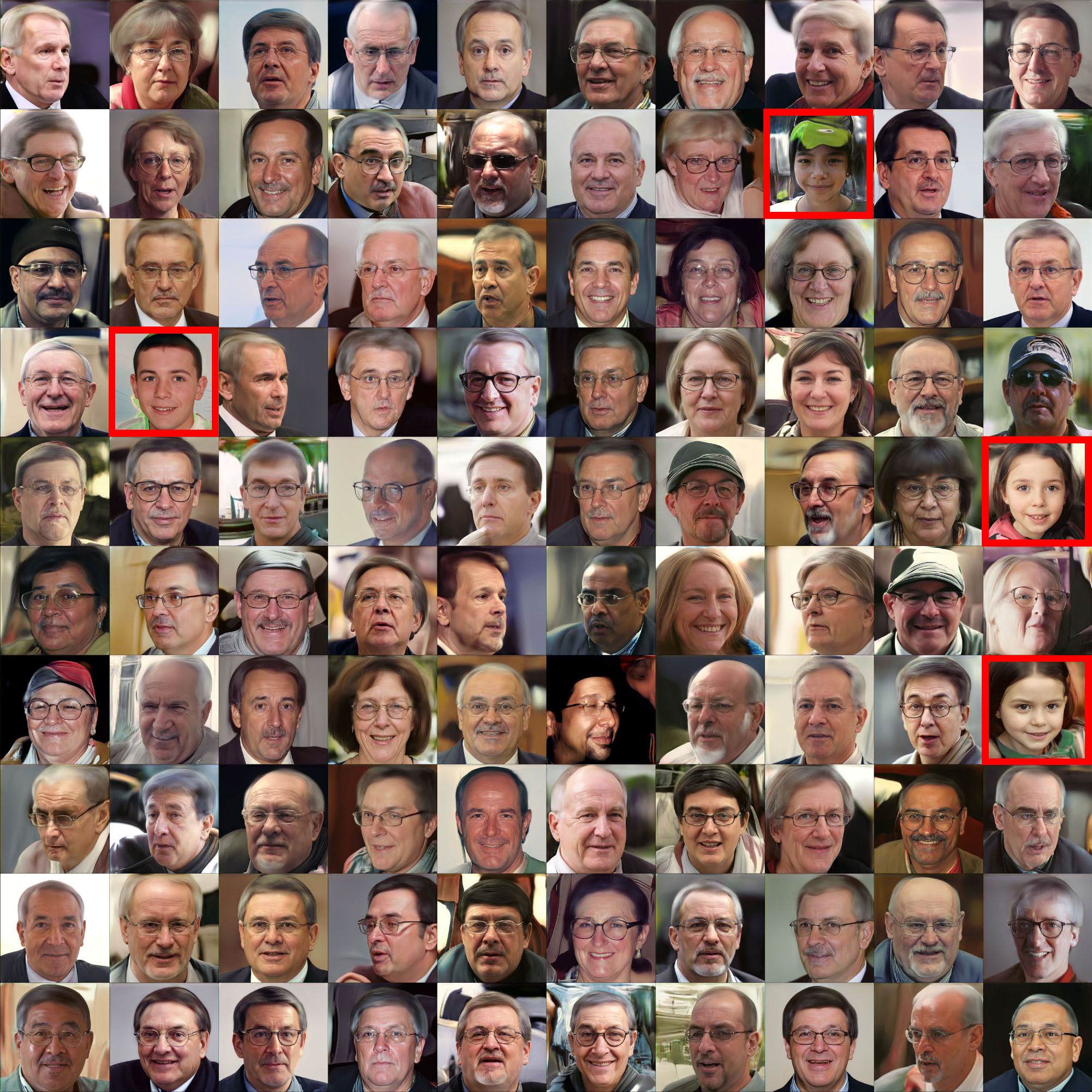}
    \caption{Uncurated samples generated by~\method in the mode coverage experiment (Sec.~\ref{subsec:other_tasks}). 4 of the 100 randomly generated images are child faces, which is close to the 5\% child face ratio in the down-sampled population.}
    \label{fig:ffhq-more-ours}
\end{figure}

\section{Limitations \& Broader impacts}\label{app:limitation}
\paragraph{Limitations.}
While~\method achieves strong one-step generation performance, several limitations remain. First, our empirical evaluation is focused on ImageNet 256$\times$256 and FFHQ; extending the framework to higher-resolution, text-conditioned, video, or other multimodal generation settings remains future work. Second, training still requires pretrained feature encoders or autoencoders in our large-scale experiments, and the choice of the feature encoders remains heuristic. Third, our theoretical convergence result analyzes the particle dynamics under regularity and asymptotic assumptions, and does not fully characterize finite-network optimization or finite-compute training in high-dimensional neural generators. 

\paragraph{Broader impacts.}
This work advances fast, high-fidelity generative modeling by reducing sampling from many iterative steps to a single generator evaluation, which can lower latency and inference cost and make generative models more accessible for creative tools, simulation, data augmentation, and domain-transfer applications. At the same time, improved image generation can also amplify risks associated with synthetic media, including impersonation, disinformation, biased representations, and privacy concerns, especially in face-related applications such as age translation. Responsible deployment should therefore include safeguards such as provenance tracking or watermarking, clear usage restrictions, dataset and model documentation, consent-aware use of face data, and monitoring or filtering mechanisms when models are released beyond research settings.

\end{document}

%% file: math_commands.tex
\usepackage{amsmath,amsfonts,bm}

\def\eqref#1{equation~\ref{#1}}

\def\1{\bm{1}}

\DeclareMathAlphabet{\mathsfit}{\encodingdefault}{\sfdefault}{m}{sl}
\SetMathAlphabet{\mathsfit}{bold}{\encodingdefault}{\sfdefault}{bx}{n}

\def\gD{{\mathcal{D}}}

\def\gF{{\mathcal{F}}}

\def\gL{{\mathcal{L}}}

\def\gP{{\mathcal{P}}}

\def\sN{{\mathbb{N}}}

\def\sR{{\mathbb{R}}}

\newcommand{\E}{\mathbb{E}}

\newcommand{\KL}{D_{\mathrm{KL}}}

\renewcommand{\eqref}[1]{\textup{(\ref{#1})}}

%% file: sec/intro.tex
\begin{figure}[htbp]
    \centering
    \includegraphics[width=0.96\linewidth]{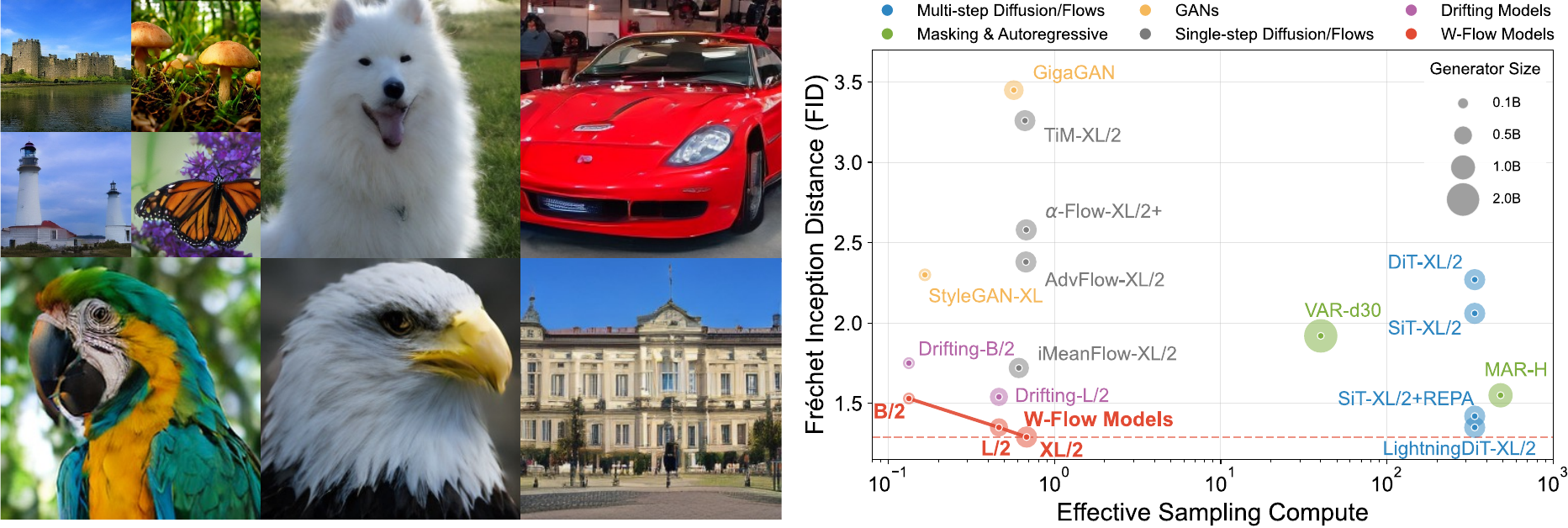}
    \caption{\emph{(Left)} \textbf{1-NFE samples} from~\method-L/2 trained from scratch on ImageNet-256$\times$256. \emph{(Right)} Sample quality (measured by FID) vs. effective sampling compute~\cite{lu2024simplifying} (billion parameters $\times$ number of function evaluations during sampling) evaluated on ImageNet 256$\times$256.}
    \label{fig:main-samples}
\end{figure}

\section{Introduction}

Generative modeling can be viewed as transporting a simple reference distribution (often referred to as a prior) to a complex target distribution (e.g., data distribution). In modern high-dimensional settings, this transport is rarely learned \emph{from scratch} as a single global map. Instead, the dominant paradigm constructs it through a sequence of simpler updates. \black{Diffusion models~\cite{sohl2015deep,song2020score,song2019generative,ho2020denoising} learn a sequence of `infinitesimal' denoising steps and then unroll them over {\it many} inference steps. Flow models~\cite{lipman2022flow,liu2022flow} learn a sequence of velocity fields and likewise unroll the dynamics over many steps.}   
Despite high sample quality, \black{multi-step inference} leads to significant latency and computational cost.

\looseness=-1
\black{It would be convenient to use an iterative procedure \emph{during training} to specify the evolution of the generated data distribution, and then compress this evolution into a static map. 
In this way, the resulting generator would directly transform an input from the reference distribution into an output from the target distribution in one step.
This would combine the efficiency of one-step generation with the flexibility of a distributional evolution during training.
The fundamental question would thus be:}
\begin{center}
\emph{What principles should govern the training dynamics of one-step generators?}    
\end{center}
Existing approaches fall short of providing a satisfactory answer. Generative Adversarial Networks (GANs)~\cite{goodfellow2014generative,radford2015unsupervised,arjovsky2017wasserstein} evolve the model via discriminator-induced gradients, yet inherit the well-known instability of minimax optimization. More recently, Drifting Model~\cite{deng2026generative} achieves strong empirical performance by iteratively displacing particles\footnote{\black{A particle is one candidate sample/image as it evolves through the algorithm.}} via handcrafted attractive and repulsive fields. Yet, the underlying dynamics remain heuristic and lack a clear and principled interpretation, leaving the model vulnerable to unpredictable convergence and mode collapse (see Fig.~\ref{fig:mode_collapse}).

{\color{black}In this paper, we propose~\method, a new framework for one-step generative modeling.~\method prescribes how the \emph{entire model distribution} should evolve toward the target distribution. This evolution is given by a Wasserstein gradient flow (WGF): at each training step, the current generated distribution moves along the steepest descent direction of a
chosen energy functional to the target distribution. We instantiate this energy functional with the Sinkhorn divergence, an optimal-transport-based distance that can be estimated efficiently from mini-batches.} The resulting Sinkhorn update combines two batch-level transport plans: a generated-to-real transport plan and a generated-to-generated self-transport plan. This makes each update depend on the batch-level transport structure, rather than on a single matched sample or nearest neighbor. We then train a generator to imitate these updates, thereby compressing the multi-step training evolution into a single one-step map at inference time. This fundamentally differs from few/one-step diffusion methods~\cite{song2023consistency,zhou2025inductive}, which
typically distill or follow a path originally designed for iterative reverse
diffusion.

Notably,~\method achieves a new state of the art for one-step generation on ImageNet 256$\times$256, obtaining 1.29 FID at XL scale and 1.35 at L scale, significantly outperforming~\cite{deng2026generative} while also surpassing many multi-step diffusion and flow models (see Fig.~\ref{fig:main-samples}). Beyond class-conditional generation, experiments on FFHQ further demonstrate improved mode coverage and strong domain transfer capability, enabled by the globally coordinated OT dynamics. Overall,~\method establishes a new framework for designing the training dynamics of generators, showing that WGFs can serve as a powerful foundation for fast and high-fidelity generative modeling.

%% file: sec/related_work.tex
\section{Related work}

\looseness=-1
\textbf{(Few/one-step) diffusion and flow-based models.} Diffusion~\cite{sohl2015deep,song2020score,ho2020denoising} and flow models~\cite{lipman2022flow,liu2022flow} learn to transport a {\color{black}simple reference distribution} 
to a {\color{black}complex target distribution} through a forward noising process and its time-reverse.
However, they require many sequential steps for sampling, leading to substantial computational cost and latency.
Many works distill multi-step teachers into few/one-step models via variational score distillation~\cite{yin2024one,yin2024improved,wang2023prolificdreamer,salimans2024multistep} or learning shortcuts along the diffusion ODE specified by the teacher~\cite{song2023consistency,song2023improved,lu2024simplifying,salimans2022progressive}. Efforts have also been made to directly learn a few/one-step generator from scratch, typically by enforcing certain self-consistency conditions on the trajectory~\cite{geng2025mean,geng2025improved,boffi2025build,song2023consistency} or the intermediate marginals~\cite{zhou2025inductive}. These methods largely inherit their training signal from a
predefined diffusion transport path, which is originally designed for iterative
generation. We instead directly define the training-time evolution of a one-step generator via WGFs.

{\color{black}\textbf{One-step generator mappings.} An alternative is to directly learn a static map that takes an input from the reference distribution and outputs a sample from the target distribution.} The most studied is the family of GANs~\cite{goodfellow2014generative,arjovsky2017wasserstein,radford2015unsupervised,genevay2018learning}, which train the generator using discriminator-driven gradients in a minimax game.
However, adversarial training is often unstable and requires considerable effort to stabilize. Recently, Drifting Model~\cite{deng2026generative} guides the generator via an empirically constructed drifting field using kernel interactions and achieves strong performance~\citep{li2026generative, gao2026drift}. Yet, the design of this technique remains largely heuristic, lacking a principled understanding of the induced dynamics. Subsequent works~\citep{lai2026unified,turan2026generative} connect Drifting Model to score-based formulations under Gaussian kernels, while~\citep{he2026sinkhorn, cao2026gradient} explore links to Sinkhorn divergence and WGFs. {\color{black}In particular, the concurrent work~\cite{he2026sinkhorn} studies a closely related Sinkhorn-drifting field and interprets Drifting Model as a Sinkhorn-divergence WGF approximation. Our focus is different: we prove convergence of the induced particle dynamics, propose large-scale training framework, and scale the method to ImageNet 256$\times$256.}

\looseness=-1
\textbf{Optimal transport in generative models.}
Existing applications of OT in generative models broadly fall into two categories. The first uses OT-inspired discrepancies as the training objectives. Wasserstein GANs~\cite{arjovsky2017wasserstein,gulrajani2017improved,tolstikhin2017wasserstein} optimize a dual form of the Wasserstein distance via an adversarial critic, while Sinkhorn GANs~\cite{genevay2018learning} leverage OT-based divergence computed by Sinkhorn iterations. The second employs OT to guide dynamic transport paths, including Schrödinger bridges~\cite{alouadi2026lightsbb,de2021diffusion,ma2025schr} and OT-inspired flow models~\citep{albergo2022building,lipman2022flow,liu2022flow,lin2025adversarial}. Yet, static OT losses require adversarial approximation
or differentiation through iterative solvers, while dynamic transport methods still rely on multi-step sampling.

\looseness=-1
\textbf{Wasserstein gradient flows in generative models.} WGFs provide a principled description of the evolution of probability measures. A standard discretization is the Jordan–Kinderlehrer–Otto (JKO) scheme~\cite{jordan1998variational}, which has inspired several recent approaches and theoretical frameworks~\citep{mokrov2021large,choi2024scalable,fan2021variational,xie2025flow}. However, JKO-based methods require solving a complex inner optimization problem at each step, often involving adversarial training over dual potentials, which makes them difficult to scale in practice.

%% file: sec/method.tex
\section{Method}\label{sec:method}

\begin{figure}[t!]
    \centering
    \includegraphics[width=0.98\linewidth]{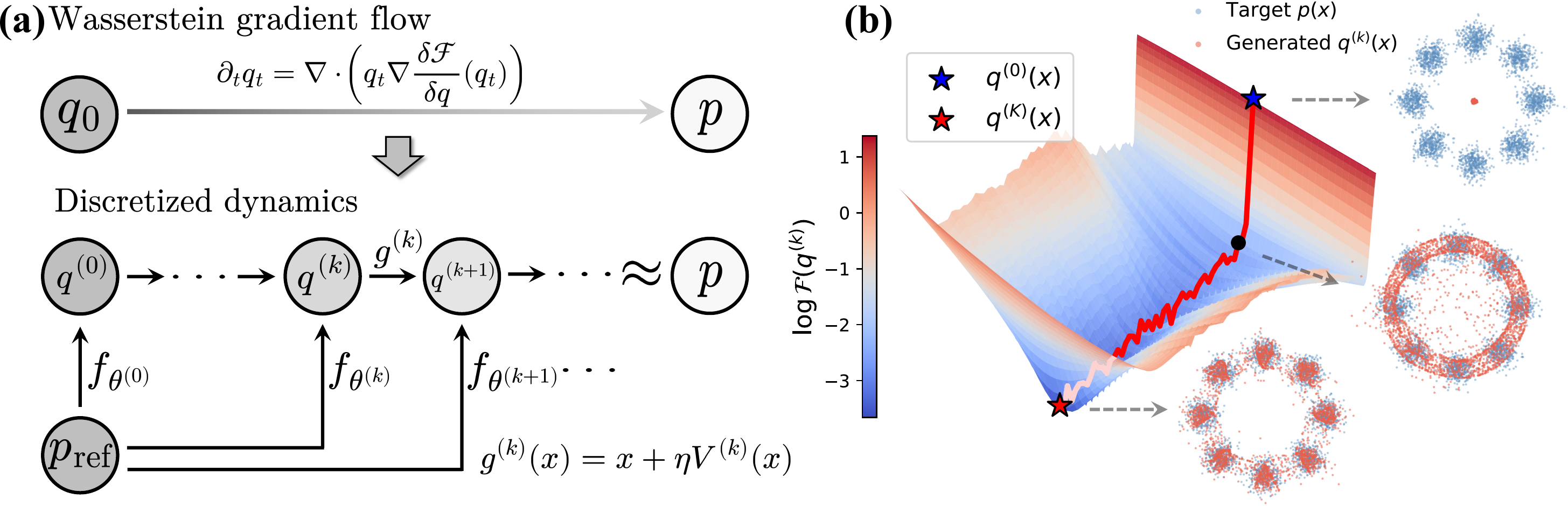}
    \caption{\textbf{(a)} The conceptual diagram of~\method. \textbf{(b)} Visualization of the training dynamics projected onto the Sinkhorn divergence landscape on 8 Gaussian mixtures, shown on a logarithmic scale.
    } 
    \label{fig:title-fig} 
\end{figure}

We introduce~\method, a {\color{black}deterministic {\it one-step} generator} with its training dynamics guided by a WGF. The main idea is to separate the design of training dynamics from the parameterization of the generator: we first prescribe how the model distribution should evolve toward the data distribution during training, and then learn a neural network to realize this evolution through particle updates.

\subsection{\method: Learning {\color{black}one-step transport map} under Wasserstein gradient flow}


{\color{black}\textbf{One-step transport maps.}}
{\color{black}Given a reference distribution $p_{\mathrm{ref}}\in\gP(\sR^m)$ and a target distribution $p\in\gP(\sR^n)$, the goal is to learn a mapping $f:\sR^m\to\sR^n$ such that if an input $z$ is sampled from the reference distribution, the output $f(z)$ follows the target distribution $p$. In distribution space, this is denoted by the pushforward operator $f_{\#}$, meaning we aim to learn a map $f$ such that $f_{\#}p_{\mathrm{ref}}\approx p$. In practice, $p$ is only observed through a finite dataset of samples.}

We adopt an incremental construction of $f$, which breaks this complex transport into a sequence of simpler steps. Specifically, let $f^{(0)}:\mathbb{R}^m\rightarrow\mathbb{R}^n$ be an initial map and set $q^{(0)}=f^{(0)}_{\#}p_{\mathrm{ref}}$. We then build the generator through a sequence of local transport maps $\{g^{(k)}\}_{k=0}^{K-1}:\mathbb{R}^n\rightarrow\mathbb{R}^n$ satisfying
\begin{align}
\label{eq:discrete-pushforward}
q^{(k+1)} = g^{(k)}_{\#} q^{(k)}, \qquad
f^{(k+1)} = g^{(k)} \circ f^{(k)}=g^{(k)}\circ \cdots\circ g^{(0)}\circ f^{(0)},
\end{align}
so that the final map $f^{(K)}$ pushes $p_{\mathrm{ref}}$ to $q^{(K)} \approx p$ (see Fig.~\ref{fig:title-fig}(a)).
Explicitly, $f^{(K)}$ is a \emph{static} map obtained by composing stepwise transport maps $\{g^{(k)}\}_{k=0}^{K-1}$ and the initial $f^{(0)}$. The role of this sequence is \emph{not} to define a standalone multi-step generative model, but rather to provide an incremental procedure for constructing the target one-step map $f^{(K)}$ used at inference time.

\textbf{Velocity field.} We parameterize the maps $g^{(k)}$ to reflect incremental \emph{displacements} between iterations:
\begin{align}
\label{eq:gt-euler}
    g^{(k)}(x)\coloneqq x+\eta V^{(k)}(x),
\end{align}
where $V^{(k)}:\sR^n\to\sR^n$ is a continuous velocity field, also referred to as a drifting field~\cite{deng2026generative}, and $\eta>0$ is the step size. Under this parameterization, each iteration moves {\color{black}a data sample $x \sim q^{(k)}$} 
only locally, while the composition of these maps yields a flexible global transport, and the field $V^{(k)}$ determines how the intermediate distributions evolve.

In principle, the recursion in \eqref{eq:discrete-pushforward} would make the generator an increasingly long composition, so exact sampling from $f^{(k)}_\# p_{\mathrm{ref}}$ becomes prohibitively costly when $k$ is large. To keep the model one-step throughout training, we represent the generator by a neural network and learn $f_{\theta^{(k+1)}}$ to match the updated map $f^{(k+1)}$. Using $f_{\theta^{(k)}}$ as a surrogate for $f^{(k)}$, we optimize 
\begin{align}
\label{eq:stepwise_loss}
    \gL^{(k)}(\theta)\coloneqq \E_{z\sim p_{\mathrm{ref}}} \left\| f_\theta(z) - g^{(k)}\circ \tilde{f}^{(k)}(z) \right\|^2,
    \qquad 
    \tilde{f}^{(k)}(z)\coloneqq \mathrm{sg}\!\left(f_{\theta^{(k)}}(z)\right),
\end{align}
where $\mathrm{sg}(\cdot)$ denotes stop-gradient. The parameter $\theta^{(k+1)}$ is then obtained by taking a gradient step in \eqref{eq:stepwise_loss}. At $\theta = \theta^{(k)}$, the loss scale is $\E_{z\sim p_{\mathrm{ref}}} \left\|\eta V^{(k)}(f_{\theta^{(k)}}(z)) \right\|^2$, hence $\gL^{(k)}=0$ implies $V^{(k)}=0$.

A central challenge, however, is to design the vector field $V^{(k)}$ to guide the iteration in \eqref{eq:discrete-pushforward} towards the target distribution. 
Existing approaches, such as~\cite{deng2026generative}, are largely heuristic and guarantee only that the target distribution is a stationary point of the dynamics: $V^{(k)} = 0$ if $q^{(k)} = p$. However, the converse generally fails.
In particular, during training, observing that $V^{(k)}\to 0$ does not imply that $q^{(k)} \to p$.
Moreover, the intermediate distributional dynamics 
$\{q^{(k)}\}$ induced by such velocity fields remain largely uncharacterized, and convergence guarantees are generally unavailable.

\textbf{Continuous-time Wasserstein velocity field.}
To obtain a principled and well-behaved velocity field $V^{(k)}$ in Eq.~\eqref{eq:gt-euler}, we take a step back and first specify the desired evolution of the model distribution in continuous time, and then derive the corresponding discrete updates. Specifically, we model the evolution of $\{q^{(k)}\}$ via a WGF.
Let $\gF:\gP(\sR^n)\to\mathbb R$ be an energy functional defined on probability distributions. The associated WGF is given by
\begin{align}
\label{eq:wgf}
    \partial_t q_t
    =
    \nabla\cdot\!\left(
    q_t \nabla \frac{\delta \gF}{\delta q}(q_t)
    \right),
\end{align}
where $\frac{\delta \mathcal{F}}{\delta q}(q_t)$ denotes the first variation of $\mathcal{F}$ at $q_t$. Eq.~\eqref{eq:wgf} defines a \emph{continuous} evolution of distributions $q_t\in\gP(\sR^n)$ that decreases the energy $\gF$. Equivalently, this evolution can be expressed through the continuity equation with a specific form of the velocity field
\begin{align}
\label{eq:velocity}
    \partial_t q_t + \nabla\cdot(q_t V_t)=0,\quad  V_t(x) = - \nabla \frac{\delta \gF}{\delta q}(q_t)(x).
\end{align}
This characterization provides a clear variational interpretation: the distribution $q_t$ evolves along the \emph{steepest descent} direction of $\gF$ in Wasserstein space (see Fig.~\ref{fig:title-fig}(b) for an illustration).
To drive the dynamics toward the target distribution $p$, we choose the energy functional $\gF(q) = \gD(q \,\|\, p)$
where $\gD$ is a suitable divergence such that $p = \arg\min_q \gF(q)$. Under this choice, the induced velocity field in \eqref{eq:velocity} transports $q_t$ to progressively reduce its discrepancy from $p$.

\textbf{Time discretization.} 
Discretizing the continuous-time evolution provides a concrete prescription for the stepwise transport maps introduced earlier. Applying an explicit Euler scheme to the WGF over time steps $\{t^{(k)}:t^{(k)}=k\eta,\;k\in\sN_{\ge 0}\}$ yields
\begin{align}
\label{eq:wgf-discrete}
    g^{(k)}(x)\coloneqq x + \eta V^{(k)}(x),
    \qquad
    V^{(k)}(x) = - \nabla \frac{\delta \gF}{\delta q}(q^{(k)})(x),
\end{align}
\looseness=-1
where $q^{(k)}=q_{t^{(k)}}$. This recovers exactly the stepwise velocity-based update in Eq.~\eqref{eq:gt-euler}, now with the velocity field $V^{(k)}$ specified by the underlying gradient flow. {\color{black}Each map $g^{(k)}$ operates directly on individual data points in $\mathbb R^n$. By moving these points, the resulting empirical distribution} approximates the continuous transport trajectory in distribution space, while composing these local updates yields a map $f^{(k)}$ that approximates the flow at time $t^{(k)}$. 
In sum, the gradient flow defines the \emph{training-time dynamics} in distribution space, whereas the learned model is ultimately a \emph{static map} obtained by compressing these local transports into a single generator.
The diagram is provided in Fig.~\ref{fig:title-fig} (a).


\subsection{Instantiation via Sinkhorn divergence}\label{subsec:algo_framework}
\textbf{Choices of the energy functional $\mathcal{F}$.} 
The choice of $\mathcal F$ is critical, as it determines both the geometry of the WGF and the tractability of the induced velocity field. Ideally, $\mathcal F$ should satisfy two properties: (i) its minimizer coincides with the target distribution $p$, and (ii) the induced flow is well-behaved and can be implemented efficiently in high dimensions.

\begin{table}[t!]
\centering
\small
\caption{Comparison of energy functionals and their induced continuous velocity fields under Wasserstein gradient flow. Detailed derivations are in Appendix~\ref{app:other}.}
\label{tab:wgf_functionals}
\begin{tabular}{lcc}
\toprule
 \textbf{Divergence $\gD$}& \textbf{Energy Functional $\mathcal{F}(q_t)$} & \textbf{Continuous Velocity Field $V_t(x)$} \\
\midrule
Squared MMD & $\frac{1}{2} \mathrm{MMD}^2(q_t, p)$ & $\int \nabla_x k(x, y) p(dy) - \int \nabla_x k(x, y) q_t(dy)$ \\
\addlinespace
KL divergence & $D_{\mathrm{KL}}(q_t \| p)$ & $\nabla \log p(x) - \nabla \log q_t(x)$ \\
\addlinespace
Sinkhorn divergence & $S_\varepsilon(q_t, p)$ & $T_{q_t,p}^\varepsilon(x) - T_{q_t,q_t}^\varepsilon(x)$~~(Eq.~\eqref{eq:continuous-sinkhorn-drift})\\
\bottomrule
\end{tabular}
\vskip -0.15in
\end{table}

Common choices include Maximum Mean Discrepancy (MMD)~\cite{li2015generative} and KL divergence (see Table~\ref{tab:wgf_functionals} for the corresponding velocity fields). However, both face significant limitations in practice. MMD induces a kernel-based interaction between the {\color{black}sampled data points}.
{\color{black}When $q_t$ is far from $p$, the kernel values can become small, causing the gradients to vanish and the process to stall.}
The KL divergence involves the score function $\nabla \log q_t(x)$, which is generally intractable and must be approximated, introducing additional bias and instability. Our ablations in Table \ref{subtab:ablation-divergence} confirm that both choices lead to inferior performance.

These limitations motivate us to seek an alternative energy functional that is computationally tractable and geometrically well-behaved. To this end, we adopt the Sinkhorn divergence, an OT-based discrepancy that avoids explicit score estimation and admits a natural and direct implementation on data samples.
While the convergence of WGFs to a global minimizer is not automatic, a recent result in \cite{hardion2026wasserstein} shows that, under Gaussian distributional assumptions, the WGF of the Sinkhorn divergence converges to its global minimizer.


\textbf{Entropic optimal transport and Sinkhorn divergence.} 
A natural candidate is the entropically regularized OT cost, which is computationally attractive due to Sinkhorn iterations. However, it is biased: in general,
$\OT_\varepsilon(p,p)\neq 0$, so  $\OT_\varepsilon(\cdot,p)$ does not define an energy whose minimizer is exactly the target distribution. This motivates the use of the \emph{debiased} Sinkhorn divergence. {\color{black}Concurrently, \cite{he2026sinkhorn} also employs Sinkhorn divergence; our implementation incorporates the two-batch self-transport and velocity guidance, which are essential to achieving strong empirical performance in large-scale training (Sec.~\ref{subsec:particle} and Sec.~\ref{subsec:implementation}).}

We start by introducing some general notations. Let \(\mathcal P_2(\mathbb R^n)\) denote the set of probability measures on
\(\mathbb R^n\) with finite second moment. Take \(q,p \in \mathcal{P}_2(\mathbb{R}^n)\). Define the set of couplings
\(
    \Pi(q,p)
:=
\big\{
\pi \in \mathcal{P}(\mathbb{R}^n \times \mathbb{R}^n)
:\ \pi(\cdot,\mathbb{R}^n)=q,\ \pi(\mathbb{R}^n,\cdot)=p
\big\},
\) which is contained in \(\mathcal P_2(\mathbb R^n \times \mathbb R^n)\) as both marginals have finite second moments. 
We then define the entropic optimal transport (EOT) functional under quadratic cost
\begin{equation}
\label{eq:OT_general}
\OT_\varepsilon(q,p)
:=
\inf_{\pi\in\Pi(q,p)}
\left\{
\int \tfrac12\|x-y\|^2\,\pi(dx,dy)
+
\varepsilon\,\KL(\pi \,\|\, q\otimes p)
\right\},
\end{equation}
where \(\varepsilon>0\) is the entropic regularization parameter.
The associated Sinkhorn divergence is given by
\begin{equation}
\label{eq:Sinkhorn_general}
S_\varepsilon(q,p)
=
\OT_\varepsilon(q,p)
-\tfrac12 \OT_\varepsilon(q,q)
-\tfrac12 \OT_\varepsilon(p,p).
\end{equation}

\textbf{Velocity field induced by Sinkhorn divergence.}
Let \(\pi_{q,p}^\varepsilon\in\Pi(q,p)\) and \(\pi_{q,q}^\varepsilon\in\Pi(q,q)\) be the optimal couplings for \(\OT_\varepsilon(q,p)\) and \(\OT_\varepsilon(q,q)\), respectively. Expressing these joint couplings in terms of their conditional distributions as $\pi_{q,p}^\varepsilon(dx,dy)=q(dx)\,\pi_{q,p}^\varepsilon(dy\mid x)$ and $\pi_{q,q}^\varepsilon(dx,dx')=q(dx)\,\pi_{q,q}^\varepsilon(dx'\mid x),$
we define the barycentric projections
\begin{equation}
\label{eq:barycentric_cont}
T_{q,p}^\varepsilon(x):=\int y\,\pi_{q,p}^\varepsilon(dy\mid x),
\qquad
T_{q,q}^\varepsilon(x):=\int x'\,\pi_{q,q}^\varepsilon(dx'\mid x).
\end{equation}
It is well known (see, \emph{e.g.},~\cite{genevay2018learning,peyre2019computational,nutz2022entropic} and the proof in Appendix~\ref{app:other}) that the velocity field coincides with the Wasserstein gradient direction of the Sinkhorn divergence
\begin{equation}
\label{eq:continuous-sinkhorn-drift}
V_{q,p}^\varepsilon(x)
=
-\nabla \frac{\delta S_\varepsilon(\cdot,p)}{\delta q}(q)(x)
=
T_{q,p}^\varepsilon(x)-T_{q,q}^\varepsilon(x).
\end{equation}
 \textbf{Characterization of equilibrium.}
    The velocity field vanishes $V_{q,p}^\varepsilon\equiv0$  if and only if $q=p$ (see \cite{feydy2019interpolating}).
    In particular, the Sinkhorn divergence admits no spurious stationary points: the only equilibrium is the target distribution. This is a notable advantage over local drifting dynamics, where the interaction field can vanish at some $q\neq p$, causing the dynamics to stall before reaching the target.

\looseness=-1
\textbf{Interpretation of the velocity field.}
The velocity field in Eq.~\eqref{eq:continuous-sinkhorn-drift} admits a natural interpretation as a balance between attraction and self-interaction. The first term \(T_{q,p}^\varepsilon(x)\) transports the current distribution toward the target distribution, while the correction term \(T_{q,q}^\varepsilon(x)\) enforces consistency with the current distribution through a global mass-conserving coupling. This structure ensures that the resulting dynamics are globally coordinated, in contrast to purely local interaction rules, and plays a key role in stabilizing the transport and preventing mode collapse (see Sec.~\ref{subsec:other_tasks} for experimental illustrations).

\subsection{Interactive particle dynamics}\label{subsec:particle}

In practice, we operate on finite mini-batches. Given $N$ i.i.d. samples $\{x_i\}_{i=1}^{N} \in \mathbb R^n$ from $q_t$ and $M$ i.i.d. samples $\{y_j\}_{j=1}^{M} \in \mathbb R^n$ from $p$,
let the weighted empirical measures be
$\widehat q_t = \sum_{i=1}^N a_i\,\delta_{x_i}$, $\widehat p = \sum_{j=1}^M b_j\,\delta_{y_j},$ with $a_i,b_j>0,$ and $\sum_i a_i=\sum_j b_j=1$.  
In this discrete setting, the general definition in Eq.~\eqref{eq:OT_general} reduces to the finite-dimensional problem over the coupling matrix $\pi \in \Pi(\widehat q_t,\widehat p)$: 
\begin{equation}
\label{eq:particle-OT}
\OT_\varepsilon \big(\widehat q_t,\widehat p\big)
:= \min_{\pi\in \Pi(\widehat q_t,\widehat p)} \ \sum_{i,j} \frac{1}{2} \|x_i - y_j\|^2 \pi_{ij} + \varepsilon \sum_{i,j}\pi_{ij}(\log \pi_{ij}-1),
\end{equation}
\begin{wrapfigure}[18]{r}{0.35\textwidth}
  \begin{center}
  \vskip -0.28in
    \includegraphics[width=0.98\linewidth]{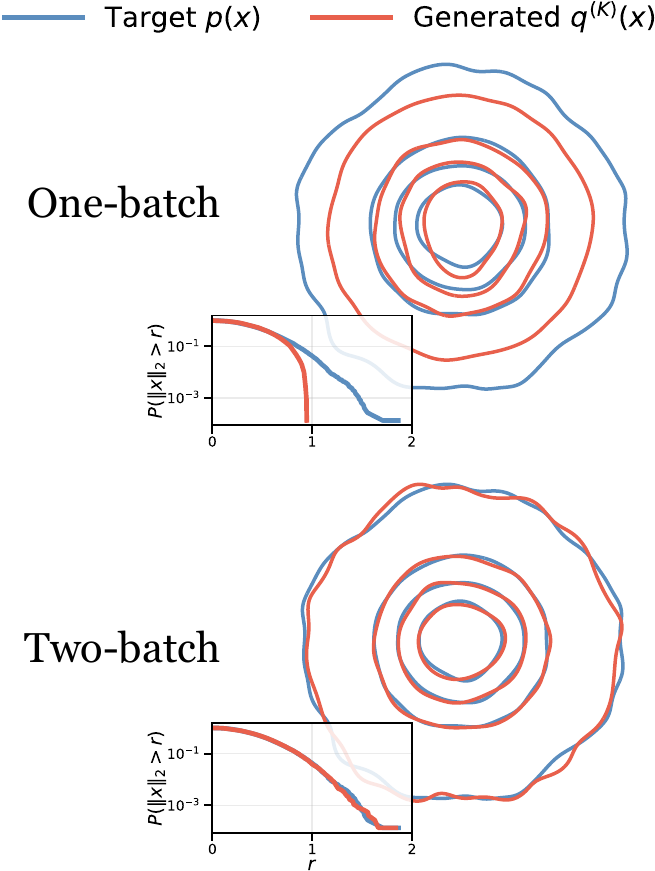}
  \end{center}
  \caption{Comparison between one-batch and two-batch estimators on learning a 2D Gaussian.}
  \label{fig:two_batch}
\end{wrapfigure}
where $\Pi(\widehat q_t,\widehat p)$ is the set of matrices with prescribed marginals. Denote the optimal solution $\pi^{\varepsilon,*}_{\widehat q_t,\widehat p}$\,.

\textbf{Two-batch estimate for self-transport.}
Na\"{\i}vely estimating the self-entropic OT term
${\rm OT}_\varepsilon(\widehat q_t,\widehat q_t)$ from a single empirical batch
introduces a self-matching artifact: since each particle can be
matched to itself at zero transport cost, the resulting coupling tends to
overemphasize the diagonal entries $\pi_{ii}$. To mitigate it, we
adopt a \emph{two-batch} strategy. We draw an independent second batch
$\{x'_l\}_{l=1}^{N}$ from $q_t$ and define the empirical measure
$\widehat q'_t = \sum_{l=1}^N a_l \delta_{x'_l}$. We then compute the coupling
between $\widehat q_t$ and $\widehat q'_t$, yielding the optimal plan
$\pi^{\varepsilon,*}_{\widehat q_t,\widehat q'_t}$. 
This construction removes the need for the diagonal masking
heuristic used in~\cite{deng2026generative} and admits a cleaner theoretical interpretation. Fig.~\ref{fig:two_batch} compares the
two-batch strategy with the na\"{\i}ve one-batch estimate on
learning a 2D Gaussian. The one-batch strategy fails to capture the Gaussian
tails, whereas the two-batch strategy recovers the target  
faithfully. Additional illustrations are in
Appendix~\ref{app:self_transport}.

The discrete counterparts of the barycentric maps in \eqref{eq:barycentric_cont} are, therefore,
\begin{equation}
\label{eq:discrete-baryproj}
    T_{\widehat q_t, \widehat p}^\varepsilon(x_i) = \frac{1}{a_i} \sum_j (\pi_{\widehat q_t, \widehat p}^{\varepsilon, *})_{ij} y_j, \quad T_{\widehat q_t, \widehat q'_t}^\varepsilon(x_i) = \frac{1}{a_i} \sum_l (\pi_{\widehat q_t, \widehat q'_t}^{\varepsilon, *})_{il} x'_l,
\end{equation}
where in practice the optimal couplings $\pi_{\widehat q_t, \widehat p}^{\varepsilon, *}$ and $\pi_{\widehat q_t, \widehat q'_t}^{\varepsilon, *}$ can be efficiently approximated using the Sinkhorn-Knopp algorithm~\cite{sinkhorn1967concerning}. We present the subroutine in Alg.~\ref{alg:sinkhorn_projection}.
With these, the particle-based estimate of the velocity field in Eq.~\eqref{eq:continuous-sinkhorn-drift} is 
\begin{align}
\label{eq:discrete-sinkhorn-drift}
   V^\varepsilon_{\widehat q_t,\widehat p}(x_i) \footnotemark  = T^\varepsilon_{\widehat q_t,\widehat p}(x_i) - T^\varepsilon_{\widehat q_t,\widehat q'_t}(x_i). 
\end{align}\footnotetext{Here we omit the second-batch $\widehat q'$ in the subscript of $V$ for brevity.}

\textbf{Convergence of the particle dynamics.} 
Now that $V^{\varepsilon}_{q,p}$ and $V^{\varepsilon}_{\widehat q, \widehat p}$ are defined, we present a convergence guarantee from the empirical dynamics to the population dynamics. This theorem
shows that the finite-particle, discrete-time transport mechanism used in our method is a consistent approximation of the intended population-level WGF. As a result, the training dynamics induced by our algorithm faithfully track the continuous distributional evolution toward the target, thereby giving a principled justification for learning the generator through these sequential local transports. See Fig.~\ref{fig:title-fig}(b) for an illustration. The convergence result holds for more general velocity fields, but we present below using the notation of Sinkhorn velocity. The full statement, assumptions, and proof are deferred to Appendix~\ref{app:proof}.

\begin{Theorem}[Informal]
\label{thm:convergence}Let \(\widehat p^M\) be the empirical measure based on \(M\) i.i.d. samples \(\{y_j\}_{j=1}^M\) from the target measure \(p\). The particles \(\{x_i^{(k)}\}_{i=1}^N\) evolve as $x_i^{(k+1)}=x_i^{(k)}+\eta\,V_{\widehat q_k^N,\widehat p^M}^\varepsilon(x_i^{(k)})$,
where \(\widehat q_k^N\) denotes the empirical measure of \(\{x_i^{(k)}\}_{i=1}^N\). Define the continuous-time interpolation as $x_i^\eta(t)
    :=
    x_i^{(k)} + (t-t_k)\,V^\varepsilon_{\widehat q_k^N,\widehat p^M}(x_i^{(k)})$,
     $t\in[t_k,t_{k+1})$, with the empirical measure
   $ \widehat q_t^{N,M,\eta}\coloneqq \frac{1}{N}\sum_{i=1}^N \delta_{x_i^\eta(t)}$.
Under suitable regularity conditions, for every $T>0$, there exists a unique
\(
q\in C([0,T];\mathcal P_2(\mathbb R^d))
\)
that solves the continuity equation
\[
 \partial_t q_t+\nabla\cdot\big(q_t\,V^\varepsilon_{q_t,p}\big)=0
\]
in the weak sense. Moreover, as $\eta\to 0$ and $N,M\to\infty$, it holds that
\[
\sup_{t\in[0,T]} \mathcal{W}_2\big(\widehat q_t^{N,M,\eta},q_t\big)\to0.
\]
\end{Theorem}
The main result is formulated for the \method setting, capturing both the two-measure
velocity field \(V^\varepsilon_{q,p}\) and the empirical target approximation. Its proof uses stability and particle-approximation techniques for nonlocal continuity equations. These techniques are largely inspired by related arguments in the literature, though the assumptions and settings are different; see \cite{piccoli2013transport}. {\color{black}Relatedly, Proposition~3.8 of \cite{he2026sinkhorn} studies the statistical approximation of the Sinkhorn drifting field, whereas our result establishes the convergence of the particle dynamics under a more general setting (see Assumption \ref{ass}).}

\algrenewcommand\algorithmicindent{0.5em}%
\begin{figure}[t]
\begin{minipage}[t]{0.49\textwidth}
\begin{algorithm}[H]
  \caption{\method Training}
  \label{alg:training_sinkhorn_drifting}
  \small
  \textbf{Input:} Generator $f_\theta$, reference distribution $p_{\mathrm{ref}}$, target distribution $p$, step size $\eta$, regularization $\varepsilon$, iterations $L$
  \begin{algorithmic}[1]
    \Repeat
    \vskip 0.03in
      \State Sample input $\{z_i\}_{i=1}^N\sim p_{\mathrm{ref}},\quad\{z'_l\}_{l=1}^N \sim p_{\mathrm{ref}}$ 
      \State Sample target $\{y_j\}_{j=1}^M \sim p$ 
      \State $x_i\leftarrow f_\theta(z_i),\quad x'_l\leftarrow\mathrm{sg}(f_\theta(z'_l))$
      \State $\{T^{\varepsilon}_{\hat q_\theta,\hat p}(x_i)\}
      \leftarrow
      \textcolor{blue}{\textsc{Sinkhorn}}(\{x_i\},\{y_j\},\varepsilon,L)$
      \State $\{T^{\varepsilon}_{\hat q_\theta,\hat q'_\theta}(x_i)\}
      \leftarrow
      \textcolor{blue}{\textsc{Sinkhorn}}(\{x_i\},\{x'_l\},\varepsilon,L)$
      \State $V^\varepsilon(x_i)
      \leftarrow
      T^{\varepsilon}_{\hat q_\theta,\hat p}(x_i)
      -
      T^{\varepsilon}_{\hat q_\theta,\hat q'_\theta}(x_i)$
      \State $\tilde x_i \leftarrow \mathrm{sg}(x_i + \eta V^\varepsilon(x_i))$
      \State Take a gradient descent step on
      \Statex  $\qquad \nabla_\theta \frac{1}{N}\sum_{i=1}^N \|f_\theta(z_i)-\tilde x_i\|^2$
      \vskip 0.04in
    \Until{converged}
  \end{algorithmic}
\end{algorithm}
\end{minipage}
\hfill
\begin{minipage}[t]{0.49\textwidth}
\begin{algorithm}[H]
  \caption{Sinkhorn Projection (\textcolor{blue}{\textsc{Sinkhorn}})}
  \label{alg:sinkhorn_projection}
  \small
  \textbf{Input:} Source particles $\{x_i\}_{i=1}^N$, target particles $\{y_j\}_{j=1}^M$, regularization $\varepsilon$, Sinkhorn iterations $L$
  \begin{algorithmic}[1]
    \State $a\leftarrow\frac{1}{N}\mathbf{1}_N,\quad b\leftarrow\frac{1}{M}\mathbf{1}_M$ \Comment{\textcolor{gray}{Uniform weights}}
    \State $C_{ij} \leftarrow \frac{1}{2}\|x_i-y_j\|^2$ \Comment{\textcolor{gray}{Cost matrix}}
    \State $K_{ij} = \exp(-C_{ij}/\varepsilon)$\Comment{\textcolor{gray}{Gibbs kernel}}
    \State $v^{(0)} \leftarrow \mathbf{1}_M$ \Comment{\textcolor{gray}{Initialization}}
    \For{$\ell=0,\dotsc,L-1$} \Comment{\textcolor{gray}{Sinkhorn-Knopp}}
      \State $u^{(\ell+1)} \leftarrow a \oslash (K v^{(\ell)})$
      \State $v^{(\ell+1)} \leftarrow b \oslash (K^\top u^{(\ell+1)})$
    \EndFor
    \State $\Pi^\varepsilon
    \leftarrow
    \operatorname{diag}(u^{(L)})\,K\,\operatorname{diag}(v^{(L)})$ \Comment{\textcolor{gray}{Coupling}}
    \State $T^\varepsilon(x_i)
    =
    \frac{1}{a_i}\sum_{j=1}^M \Pi^\varepsilon_{ij} y_j$ \Comment{\textcolor{gray}{Projection}}
    \State \textbf{return} $\{T^\varepsilon(x_i)\}_{i=1}^N$
  \end{algorithmic}
\end{algorithm}
\end{minipage}
\vspace{-1.5em}
\end{figure}


\subsection{Implementation}\label{subsec:implementation}

\textbf{Training.} Empirically, we construct the training loss by learning the generator $f_\theta$ to simulate the discrete particle dynamics driven by the velocity in Eq.~\eqref{eq:discrete-sinkhorn-drift}. This can be viewed as a finite-sample counterpart of the population-level objective in \eqref{eq:stepwise_loss}, where expectations with respect to \(q^{(k)}\) and \(p\) are replaced by empirical measures. Specifically, at each training step, we optimize
\begin{equation}\label{eq:ot-loss}
\mathcal{L}_{\text{\method}}(\theta)
= \frac{1}{N}\sum_{i=1}^N \left\|x_i - \mathrm{sg}\bigl(x_i + \eta\,V^\varepsilon_{\widehat q_\theta,\widehat p}(x_i)\bigr)\right\|^2,
\quad x_i=f_\theta(z_i),
\end{equation}
where $\{z_i\}_{i=1}^N$ are i.i.d. samples from $p_{\mathrm{ref}}$.
Crucially, the stop-gradient operator $\mathrm{sg}(\cdot)$ precisely aligns with the discrete WGF dynamics by treating the regression target as fixed, which also avoids differentiation through the OT plan and stabilizes the optimization.
We summarize the overall training procedure in Alg.~\ref{alg:training_sinkhorn_drifting}.

\textbf{Classifier-free guidance.} Our framework naturally supports baking classifier-free guidance (CFG)~\cite{ho2022classifier} into the pushforward map $f_\theta$ at \emph{training time}. The idea is to use a small batch of unconditional samples as additional reference particles, thereby introducing an extra repulsive effect on the generated samples. 
A natural implementation in~\cite{deng2026generative} is to modify the repulsive distribution as
\begin{align}
\label{eq:dist-cfg}
    \tilde{q}^w_\theta(x|c) \coloneqq \frac{1}{w+1}q_\theta(x|c)+\frac{w}{w+1}p(x|\varnothing),
\end{align}
where $w\geq 0$ is the guidance scale and $p(x|\varnothing)$ is the unconditional (marginal) distribution. Operationally, this amounts to sampling additional particles from $p(x|\varnothing)$ with importance weight $w$ \emph{w.r.t.} $q_\theta$. 
However, this approach implicitly {\it alters} the target being learned. One can show that it corresponds to learning a \emph{linearly} extrapolated target distribution $\tilde{p}^w(x|c)= p(x|c) + w(p(x|c) - p(x|\varnothing))$, which deviates from the standard CFG formulation based on exponential tilting~\cite{ho2022classifier,dhariwal2021diffusion}.

\looseness=-1
In light of this, we instead introduce guidance directly at the \emph{velocity field level}. Specifically, we define
\begin{align}
\label{eq:velocity-cfg}
    \tilde{V}^{\varepsilon,w}_{q_\theta,p}(x_i)\!\coloneqq\! \left(T^\varepsilon_{q_\theta(\cdot|c),p(\cdot|c)}(x_i) \!-\! T^\varepsilon_{q_\theta(\cdot|c),q_\theta(\cdot|c)}(x_i)\right) \!+\! w \underline{\left(T^\varepsilon_{q_\theta(\cdot|c),p(\cdot|c)}(x_i) \!-\! T^\varepsilon_{q_\theta(\cdot|c),p(\cdot|\varnothing)}(x_i)\right)},
\end{align}
where $T^\varepsilon_{q_\theta(\cdot|c),p(\cdot|\varnothing)}$ denotes the Sinkhorn barycentric projection (Alg.~\ref{alg:sinkhorn_projection}) between $q_\theta(\cdot|c)$ and the unconditional data distribution $p(\cdot|\varnothing)$. This construction injects guidance via the underlined term that computes the difference between the conditional and unconditional velocity fields (see Appendix~\ref{app:kl_velocity_guidance}).
Due to these conceptual differences, we refer to Eq.~\eqref{eq:dist-cfg} as \emph{distribution guidance} and to Eq.~\eqref{eq:velocity-cfg} as \emph{velocity guidance}. Although the effective target distribution induced by Eq.~\eqref{eq:velocity-cfg} is generally not available in closed form, it admits a natural geometric interpretation through a barycentric construction. Moreover, we show in Appendix~\ref{app:kl_velocity_guidance} that, when the KL divergence is used as $\gF$, this formulation exactly recovers the standard exponentially tilted target. This provides a conceptual justification for adopting the same design under the Sinkhorn divergence, whose empirical benefits are demonstrated in Sec.~\ref{exp:abl}.

\looseness=-1
\textbf{Sinkhorn divergence in feature space.}
Rather than evaluating the Sinkhorn divergence directly in the target space $\sR^n$, we may first lift the particles $\{x_i\},\{x'_l\},\{y_j\}$ into a feature space through a feature map
$\phi:\mathbb{R}^n \rightarrow \mathbb{R}^d$. The velocity field and training loss are then computed using the embedded particles $\{\phi(x_i)\}, \{\phi(x'_l)\}, \{\phi(y_j)\}$, while gradients are backpropagated through $\phi$ to update the generator. In practice, following~\cite{deng2026generative}, we use multiple feature maps and sum the corresponding losses.

%% file: sec/experiments.tex
\section{Experiments}\label{sec:experiments}

\textbf{Experiment setup.} We conduct our main experiments on ImageNet~\cite{deng2009imagenet} generation with 256$\times$256 resolution. We directly adopt the DiT-style~\cite{peebles2023scalable} generator architecture in~\cite{deng2026generative}. We also use the officially released pretrained MAE feature extractors from~\cite{deng2026generative}. More details are deferred to Appendix~\ref{app:implementation}.


\subsection{Ablation study}\label{exp:abl}

\textbf{Ablation setting.}
To enable an informative comparison in a controlled setting, we follow the default ablation setup of~\cite{deng2026generative}. Specifically, all ablation experiments are conducted in the SD-VAE~\cite{rombach2022high} latent space, using a B/2 DiT generator together with the pretrained latent-MAE feature encoder from~\cite{deng2026generative}. The generator is trained for 100 epochs, and FID~\cite{fid} is evaluated on 50K generated images. We set $\varepsilon=0.05$ by default with additional detailed hyperparameters provided in Appendix~\ref{app:implementation}.

\textbf{Energy functionals in gradient flow.} We implement and compare three different energy functionals: squared MMD, KL divergence, and Sinkhorn divergence (Eq.~\eqref{eq:Sinkhorn_general}). Results in Table~\ref{subtab:ablation-divergence} confirm that Sinkhorn divergence performs the best, whereas both MMD and KL divergence yield noticeably weaker results, highlighting the advantage of the OT-driven dynamics. The Sinkhorn-based gradient flow also significantly outperforms the heuristic design of Drifting Model when both use distribution guidance.

\textbf{Classifier-free guidance.} We further equip the Sinkhorn~\method with our proposed velocity-guidance CFG (denoted as \texttt{velo.}). As reported in Table~\ref{subtab:ablation-cfg}, this guidance mechanism is fully compatible with Sinkhorn~\method and further improves the FID to 7.08, outperforming distribution guidance (\texttt{dist.}). We adopt it as the default configuration in the subsequent studies.

\textbf{Entropy regularization parameter $\varepsilon$.} We next study the effect of the parameter  $\varepsilon$ in the EOT problems; the results are reported in Table~\ref{subtab:ablation-cfg}. Overall,~\method is fairly robust to the choice of $\varepsilon$. Among the values we consider, $\varepsilon=0.05$ yields the best performance under the ablation setting. See Table~\ref{tab:more-epsilon-distance} for more details. For the main experiments, we set $\varepsilon=0.05$ by default, with detailed choices specified in Table~\ref{tab:hypers_all}.

\renewcommand{\thesubfigure}{{\alph{subfigure}}}
\begin{table*}[t!]
\centering
\begin{minipage}{0.99\linewidth}
\subfloat[
\textbf{Choice of divergence}. Sinkhorn divergence leads to the lowest FID, outperforming heuristic designs in~\cite{deng2026generative}.
]{
\begin{minipage}{0.30\linewidth}
\tablestyle{3pt}{1.02}
\begin{tabular}{x{40}x{44}}
Divergence & FID, 1-NFE \\
\hline
Drifting~\cite{deng2026generative} & ~~8.46 \\
\hline
MMD & 10.40 \\
KL & 10.17 \\
Sinkhorn  & ~~\textbf{7.29} \\
\end{tabular}
\label{subtab:ablation-divergence}
\end{minipage}
}
\hfill
\subfloat[\textbf{Ablation on CFG and $\varepsilon$}. Our proposed \texttt{velo} CFG with $\varepsilon=0.05$ for Sinkhorn divergence performs the best.
]{
\begin{minipage}{0.30\linewidth}
\tablestyle{3pt}{1.02}
\begin{tabular}{x{24}x{24}x{44}}
 $\varepsilon$ & CFG & FID, 1-NFE \\
\hline
0.05 & \texttt{dist.} & 7.29 \\
 \cellcolor{gray!20}0.05 &  \cellcolor{gray!20}\texttt{velo.} &  \cellcolor{gray!20}\textbf{7.08} \\
0.025 & \texttt{velo.} & 7.11 \\
0.075 & \texttt{velo.} & 7.26 \\
\end{tabular}
\label{subtab:ablation-cfg}
\end{minipage}
}
\hfill
\subfloat[
\textbf{OT cost and $V_t$ estimation.} ``Mask''  and ``2-batch'' refer to the \emph{diagonal mask} and \emph{two-batch estimator} in $T^\varepsilon_{q,q}$~(\textsection~\ref{subsec:particle}).
]{
\begin{minipage}{0.32\linewidth}
\tablestyle{2.5pt}{1.02}
\begin{tabular}{x{23}x{18}x{28}x{44}}
Cost & Mask & 2-batch & FID, 1-NFE \\
\hline
$\|\cdot\|$ & - & \checkmark & ~~7.16 \\
\cellcolor{gray!20}$\frac{1}{2}\|\cdot\|^2$ & \cellcolor{gray!20}- & \cellcolor{gray!20}\checkmark & \cellcolor{gray!20} ~~\textbf{7.08} \\
$\frac{1}{2}\|\cdot\|^2$ & - & -  & 17.57 \\
$\frac{1}{2}\|\cdot\|^2$ & \checkmark & - & ~~7.45 \\
\end{tabular}
\label{subtab:ablation-maskdiag}
\end{minipage}
}
\end{minipage}
\caption{\textbf{Ablation study on ImageNet 256$\times$256 generation.} We use a B/2 backbone with 100-epoch training from scratch. The FID on 50K images is reported. The default settings are marked in \colorbox{gray!20}{gray}.
}
\label{tab:fid_ablation}
\end{table*}

\textbf{The cost function in OT.} We compare our default choice of quadratic cost $c(x,y)=\frac{1}{2}\|x-y\|^2$ with the $\ell_2$ distance cost $c(x,y)=\|x-y\|$. The results in Table~\ref{subtab:ablation-maskdiag} demonstrate that the quadratic cost yields superior performance, in line with our theoretical assumptions in Eq.~\eqref{eq:OT_general}. 

\looseness=-1
\textbf{Velocity field estimation.} We validate the two-batch estimator in Sec.~\ref{subsec:particle} for the self-transport map. As shown in Table~\ref{subtab:ablation-maskdiag},  estimating  $T^\varepsilon_{\widehat q_\theta,\widehat q_\theta}$ from a single batch leads to a substantially worse FID of 17.57, due to the intrinsic self-matching bias of this estimator. Our two-batch strategy also outperforms the mask-diagonal heuristic used in~\cite{deng2026generative}, while admitting a principled theoretical interpretation.


\subsection{Main results on ImageNet}

\textbf{Setup.} Putting together the key components validated in the ablation study, we scale up training to achieve state-of-the-art results on ImageNet 256$\times$256 generation. We follow the setup of~\cite{deng2026generative} with the SD-VAE latent space, adopting the same generator architecture and the same pretrained feature encoder weights. We also largely retain their hyperparameter configuration (see Appendix~\ref{app:implementation}).

\input{tables/main_table}

\textbf{Results.}
We compare~\method with state-of-the-art one-step generators and multi-step diffusion and flow models. Results are reported in Table~\ref{tab:in256_latent}. Notably,~\method establishes a new state of the art for one-step class-conditional generation on ImageNet 256$\times$256, achieving an FID of 1.29 at XL scale and 1.35 at L scale, outperforming MeanFlow-based methods and Drifting Models by a clear margin. Remarkably,~\method B/2, with only 133M generator parameters, achieves an FID of 1.52, surpassing Drifting Model L/2 despite its substantially larger 463M-parameter generator.
Furthermore, despite 1-NFE sampling,~\method outperforms most diffusion models requiring up to 250 steps, such as LightningDiT-XL/2. See Appendix~\ref{app:throughput} for a detailed inference throughput comparison. These strong empirical results support our central claim that principled WGF dynamics can translate into exceptional generation performance.~\method also exhibits a favorable scaling behavior with model size. More image samples produced by~\method are provided in Appendix~\ref{app:generated_samples}.

\begin{figure}[t!]
    \centering
    \includegraphics[width=0.98\linewidth]{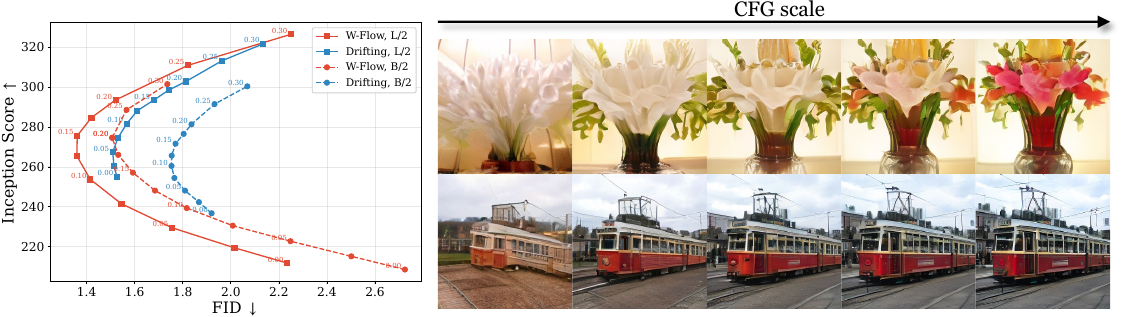}
    \caption{\textbf{Classifier-free guidance.} \emph{Left:} The FID and Inception Score curve when sweeping over CFG scales. 
    \emph{Right:} Image samples by~\method, L/2 with CFG increasing from 0.0 to 2.0.}
    \label{fig:cfg-sweep}
\end{figure}

\textbf{The effect of CFG.} 
Fig.~\ref{fig:cfg-sweep} shows the FID and Inception Score with different CFG scales at inference time;~\method clearly achieves a better tradeoff compared with Drifting~\cite{deng2026generative}. Samples in Fig.~\ref{fig:cfg-sweep} demonstrate that increasing CFG leads to clearer class-specific details and higher perceptual quality, which aligns with the expected CFG behavior.

\begin{wraptable}[6]{r}{0.33\textwidth}
  \centering
  \small
  \vskip -0.2in
 \setlength{\tabcolsep}{3pt}
  \caption{FID vs. training epoch.}
      \resizebox{1.0\linewidth}{!}{
    \begin{tabular}{lc|cc}
          & Epoch & B/2   & L/2 \\
    \hline
    \rowcolor[gray]{0.9}
    Drifting~\cite{deng2026generative} & 1280  & 1.75  & 1.54 \\
    \multirow{3}[1]{*}{\method} & 384   & 1.75  &  1.55\\
          & 640   & 1.63  &  1.46 \\
          & 1280  & 1.52  & 1.35 \\
    \end{tabular}%
    }
  \label{tab:convergence}%
\end{wraptable}%
\textbf{Convergence speed.}~\method induces training dynamics that follow the steepest descent of the energy functional, 
leading to faster convergence at training time. We empirically validate this effect in Table~\ref{tab:convergence},
where \method converges substantially faster than Drifting Model. In particular, both \method B/2 and L/2 trained for only 384 epochs closely match the FID of Drifting Model trained for 1280 epochs.

\subsection{Domain transfer and mode coverage}\label{subsec:other_tasks}

\textbf{Domain transfer.} 
Unlike standard generative models that typically map Gaussian noise to data, ~\method can instead define a WGF between arbitrary distributions. This makes it naturally suited for domain transfer tasks where the source and target share the same semantic space. By setting $p_{\mathrm{ref}}$ to be the source data distribution, it learns a one-step domain-transfer map, \emph{e.g.}, from senior to young-adult faces. To preserve source semantics, we parameterize the map as a residual network~\cite{he2016deep}, $f_\theta(z)=z+\mathcal{N}_\theta(z)$, and zero-initialize the final layer of $\mathcal{N}_\theta$, such that $f_{\theta_0}(z)=z$.

In a 2D oval-to-circle toy example in Fig.~\ref{fig:ffhq}(a), Drifting Model shows erratic intermediate states and often moves particles far from their initial positions. In contrast, the OT-induced velocity in~\method produces coordinated, nearly straight, and much shorter trajectories toward the target.

\looseness=-1
We further evaluate one-step age translation on FFHQ~\cite{karras2019style}, mapping senior faces (ages 55--100) to young-adult faces (ages 18--30) in the 512-dimensional latent space of a pretrained autoencoder~\cite{pidhorskyi2020adversarial}. We train a 4-layer MLP that maps source latent codes to the target domain. Fig.~\ref{fig:ffhq}(b) and~\ref{fig:ffhq}(c) show that~\method achieves shorter latent $\ell_2$ transport distances and better preserves source identity.

\begin{figure}[t!]
    \centering
    \includegraphics[width=1.0\linewidth]{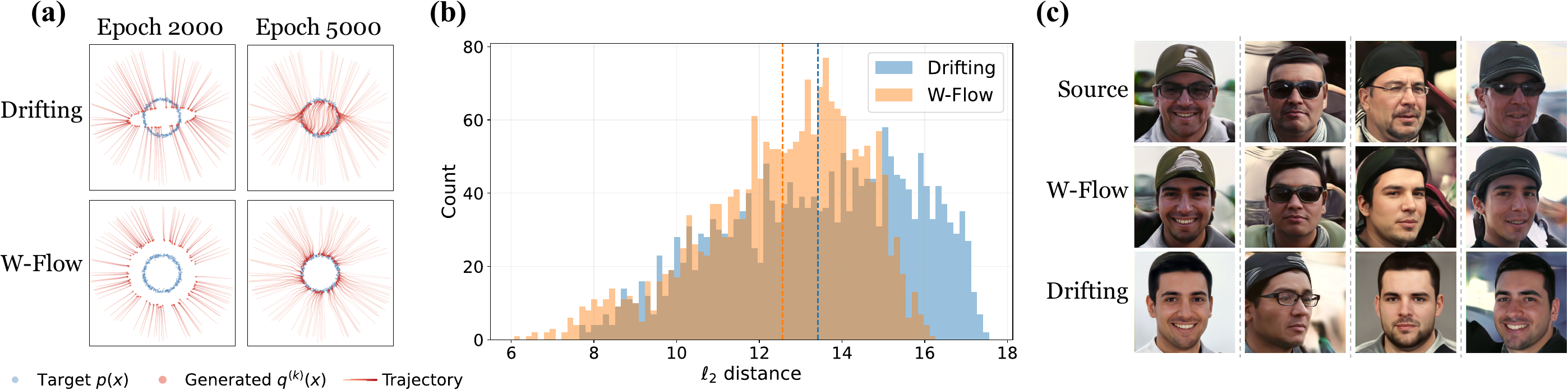}
    \caption{\textbf{(a) Oval-to-circle domain transfer.} Source and target are constructed by sampling angles uniformly from $[0, 2\pi)$ with parametric curves corrupted by Gaussian noise. \textbf{(b) \& (c) One-step facial age translation on FFHQ, mapping older faces to younger ones.} \textbf{(b)} Histogram of the latent $\ell_2$ distance  between 2,000 source images and their generated targets.
    \textbf{(c)} Visual comparison. 
    }
    \label{fig:ffhq}
    \vskip -0.1in
\end{figure}

\begin{figure}[t]
    \centering
    \includegraphics[width=0.95\linewidth]{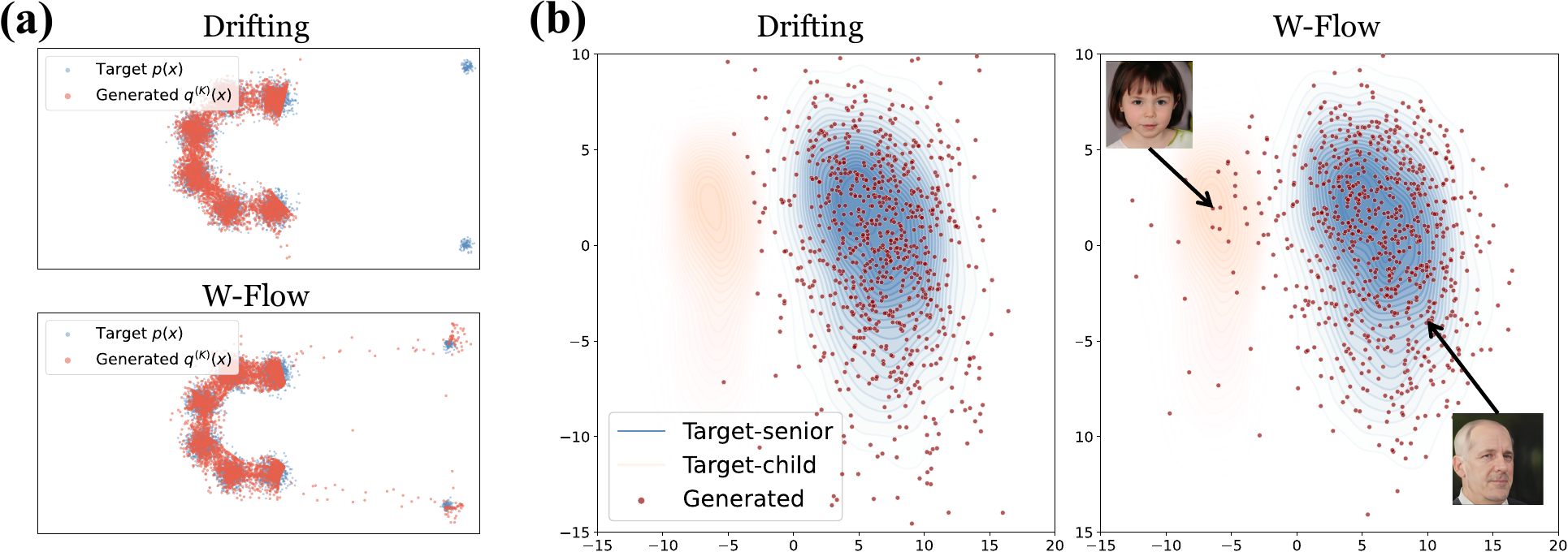}
    \caption{\textbf{Evaluation of mode coverage under imbalanced target distributions.} \textbf{(a)} Evaluation of mode coverage on a 2D Gaussian mixture dataset featuring six dominant modes and two distant minority modes. \textbf{(b)} PCA scatter plot of generated latent codes for an artificially imbalanced FFHQ target distribution (95\% senior faces, 5\% child faces). See Appendix \ref{app:generated_samples} for generated samples comparing mode coverage behavior.
    }
    \label{fig:mode_collapse}
    \vskip -0.15in
\end{figure}

\textbf{Mode coverage.}
We further investigate whether the globally structured OT plan in~\method can mitigate mode collapse, a known vulnerability of Drifting Model (see Appendix~\ref{subsec:discussion}). We construct a 2D Gaussian mixture dataset with six large, tightly concentrated modes and two small distant minority modes. As shown in Fig.~\ref{fig:mode_collapse}(a), Drifting Model greedily collapses on the dominant clusters, while the Sinkhorn barycentric projections in~\method instead enforce global marginal constraints and successfully preserve the coverage of all modes.

We further extend this observation to image generation on FFHQ under an imbalanced training distribution, where target batches contain 95\% senior faces (ages 55--100) and 5\% child faces (ages 0--12). The PCA visualization in Fig.~\ref{fig:mode_collapse}(b) shows that Drifting Model almost entirely drops the minority child mode, whereas~\method successfully captures it while maintaining high fidelity on the dominant senior mode. See Appendix~\ref{app:generated_samples} for representative samples.

%% file: tables/main_table.tex
\newcommand\headspace{\hspace{.2em}}
\colorlet{mygray}{black!40}
\newcommand{\cg}{\color{mygray}}
\begin{table*}[t]
    \centering
\caption{\textbf{Class-conditional generation on ImageNet 256$\times$256.}
        FID is computed on 50K images. All methods are reported with CFG if applicable. ``\#Params'' denotes the number of parameters of ``generator + decoder''. All generators are trained from scratch without distilling from a teacher model.}
        \label{tab:in256_latent}
        \vskip -0.05in
    \begin{minipage}[t]{0.47\textwidth}
        \centering
        \tablestyle{2pt}{1.02}
        \tablefontsize
         \resizebox{\linewidth}{!}{
        \begin{tabular}{@{}l c c r r@{}}
        \toprule
          & \#Params & NFE & FID$\downarrow$ & IS$\uparrow$ \\
        \midrule
        \rowcolor[gray]{0.9} \multicolumn{5}{l}{\textit{Multi-step Diffusion/Flows}} \\
        \headspace \cg ADM-G \tiny\cite{dhariwal2021diffusion}  & \cg \tiny 554M & \cg \tiny 250$\times$2 & \cg 4.59 & \cg 186.7 \\
        \headspace \cg DiT-XL/2 \tiny\cite{peebles2023scalable}  & \cg \tiny 675M+49M & \cg \tiny 250$\times$2 & \cg 2.27 & \cg 278.2 \\
        \headspace \cg SiT-XL/2 \tiny\cite{ma2024sit}  & \cg \tiny 675M+49M & \cg \tiny 250$\times$2 & \cg 2.06 & \cg 270.3 \\
        \headspace \cg SiT-XL/2+REPA \tiny\cite{yu2024representation} & \cg \tiny 675M+49M & \cg \tiny 250$\times$2 & \cg 1.42 & \cg 305.7 \\
        \headspace \cg LightningDiT-XL/2 \tiny\cite{yao2025reconstruction} & \cg \tiny 675M+70M & \cg \tiny 250$\times$2 & \cg 1.35 & \cg 295.3 \\
        \headspace \cg RAE+DiT$^{\text{DH}}$-XL/2 \tiny\cite{zheng2025diffusion} & \cg \tiny 839M+415M & \cg \tiny 50$\times$2 & \cg \textbf{1.13} & \cg 262.6 \\
        \midrule
        \rowcolor[gray]{0.9} \multicolumn{5}{l}{\textit{Masking \& Autoregressive}} \\
        \headspace \cg MaskGIT \tiny\cite{chang2022maskgit}  & \cg \tiny 227M+31M & \cg \tiny 8 & \cg 6.18 & \cg 182.1 \\
        \headspace \cg VAR-$d30$ \tiny\cite{tian2024visual}  & \cg \tiny 2B+61M & \cg \tiny 10$\times$2 & \cg 1.92 & \cg 323.1 \\
        \headspace \cg MAR-H \tiny\cite{li2024autoregressive}  & \tiny \cg 943M+41M & \cg \tiny 256$\times$2 & \cg 1.55 & \cg 303.7 \\
        \midrule
        \rowcolor[gray]{0.9} \multicolumn{5}{l}{\textit{GANs}} \\
        \headspace BigGAN \tiny\cite{brock2018large}  & \tiny 112M & \tiny 1 & 6.95 & 152.8 \\
        \headspace GigaGAN \tiny\cite{kang2023scaling}  & \tiny 569M & \tiny 1 & 3.45 & 225.5 \\
        \headspace StyleGAN-XL \tiny\cite{sauer2022stylegan}  & \tiny 166M & \tiny 1 & 2.30 & 265.1 \\
        \bottomrule
        \end{tabular}
        }
    \end{minipage}
    \hfill
    \begin{minipage}[t]{0.49\textwidth}
        \centering
                \tablestyle{2pt}{1.02}
        \tablefontsize
         \resizebox{\linewidth}{!}{
        \begin{tabular}{@{}l c c r r@{}}
        \toprule
                  & \#Params & NFE & FID$\downarrow$ & IS$\uparrow$ \\
        \midrule
        \rowcolor[gray]{0.9} \multicolumn{5}{l}{\textit{Single-step Diffusion/Flows}} \\
        \headspace iCT-XL/2 \tiny\cite{song2023improved}  & \tiny 675M+49M & \tiny 1 & 34.24 & -- \\
        \headspace Shortcut-XL/2 \tiny\cite{frans2024one} & \tiny 675M+49M & \tiny 1 & 10.60 & -- \\
        \headspace MeanFlow-XL/2 \tiny\cite{geng2025mean} & \tiny 676M+49M & \tiny 1 & 3.43 & -- \\
        \headspace TiM-XL/2 \tiny\cite{wang2025transition} & \tiny 664M+49M & \tiny 1 & 3.26 & 210.3 \\
        \headspace $\alpha$-Flow-XL/2+ \tiny\cite{zhang2025alphaflow} & \tiny 676M+49M & \tiny 1 & 2.58 & -- \\
        \headspace AdvFlow-XL/2 \tiny\cite{lin2025adversarial} & \tiny 673M+49M & \tiny 1 & 2.38 & 284.2 \\
        \headspace iMeanFlow-XL/2 \tiny\cite{geng2025improved} & \tiny 610M+49M & \tiny 1 & 1.72 & 282.0 \\
        \midrule
        \rowcolor[gray]{0.9} \multicolumn{5}{l}{\textit{Drifting Models}} \\
        \headspace Drifting Model, B/2 \tiny\cite{deng2026generative}  & \tiny 133M+49M & \tiny 1 & 1.75 & 263.2 \\
        \headspace Drifting Model, L/2 \tiny\cite{deng2026generative} & \tiny 463M+49M & \tiny 1 & 1.54 & 258.9 \\
        \midrule
        \headspace \textbf{\method, B/2} & \tiny 133M+49M & \tiny 1 & \textbf{1.52} & 271.8 \\
        \headspace \textbf{\method, L/2}  & \tiny 463M+49M & \tiny 1 & \textbf{1.35} & 272.5  \\
        \headspace \textbf{\method, XL/2}  & \tiny 679M+49M & \tiny 1 & \textbf{1.29} & 265.4  \\
        \bottomrule
        \end{tabular}
    }
    \end{minipage}
    \vskip -0.2in
\end{table*}

%% file: sec/conclusion.tex
\section{Conclusion}

We introduced~\method, a new paradigm for one-step generative modeling that guides the training dynamics of image generators by Wasserstein gradient flow. Our method derives the velocity field via the steepest descent of Sinkhorn divergence, yielding a globally coordinated OT-based update rule with consistent particle dynamics. Empirically,~\method sets a new state of the art for one-step ImageNet 256$\times$256 generation with 1.29 FID.
The results establish the WGF dynamics as a foundation for fast, high-fidelity generation. Future work includes exploring a broader class of energy functionals within the WGF framework and extending this paradigm beyond class-conditional image generation to more complex settings such as text-to-image generation.